\documentclass[11pt]{article}

\usepackage[final]{acl}

\usepackage{times}
\usepackage{latexsym}

\usepackage[T1]{fontenc}

\usepackage[utf8]{inputenc}

\usepackage{microtype}

\usepackage{inconsolata}

\usepackage{graphicx}

\usepackage[ruled,linesnumbered,vlined]{algorithm2e}
\usepackage[normalem]{ulem}
\usepackage{bm}
\usepackage{amsthm}
\usepackage{amsmath}
\usepackage{spverbatim}
\usepackage{booktabs}
\usepackage{fontawesome5}
\usepackage{amssymb}
\usepackage{marvosym}
\usepackage{xcolor}
\usepackage[most]{tcolorbox}
\usepackage{colortbl}
\usepackage{float}
\usepackage{wrapfig}
\usepackage{adjustbox}
\usepackage{multirow}
\usepackage{array}
\usepackage{titletoc}
\usepackage{tikz}
\usepackage{pifont}
\usepackage{makecell}
\usepackage{enumitem}
\usetikzlibrary{arrows.meta}

\SetKwIF{If}{ElseIf}{Else}{if}{then}{else if}{else}{}
\SetKwFor{ForEach}{for each}{do}{}

\newtheorem{definition}{Definition}
\newtheorem{proposition}{Proposition}

\newcommand{\ourmodel}{\textsc{HCR-Reasoner}\xspace}
\newcommand{\ourbench}{\textsc{HCR-Bench}\xspace}
\newcommand{\appendYes}{$\mathcal{A}.\text{append}($``Yes''$)$}

\title{Synergizing Large Language Models and Theory for \\ Human-like Causal Reasoning}

\author{%
    Yanxi Zhang\textsuperscript{1,2},
    Xin Cong\textsuperscript{3},
    Zhong Zhang\textsuperscript{3~\faEnvelope},
    Xiao Liu\textsuperscript{2},
    Dongyan Zhao\textsuperscript{2,4,5~\faEnvelope},
    Yesai Wu\textsuperscript{3} \\
    \textsuperscript{1}Center for Data Science, AAIS, Peking University \\
    \textsuperscript{2}Wangxuan Institute of Computer Technology, Peking University \\
    \textsuperscript{3}Tsinghua University
    \textsuperscript{4}State Key Laboratory of General Artificial Intelligence \\
    \textsuperscript{5}Key Laboratory of Intelligent Press Media Technology \\
    \texttt{zhangyx@stu.pku.edu.cn}, 
    \texttt{zhongzhang@tsinghua.edu.cn},
    \texttt{zhaodongyan@pku.edu.cn} \\
    \faGithub~\url{https://github.com/zhangyx0417/hcr_reasoner}
}

\begin{document}
\maketitle

\begin{abstract}
Performing causal reasoning akin to humans is fundamental to achieving strong artificial intelligence.
However, existing studies on causal reasoning and large language models (LLMs) often underestimate the role of \textit{human causal cognition}.
Human causal reasoning is far more than recognizing causes and effects. Humans often factor in many other considerations when establishing a causal link between two events, such as norm violations, intentionality, morality, counterfactual scenarios, etc.
To simulate such reasoning, it is intuitive to identify cause-effect pairs using \textit{token causality} formalisms while modeling the factors that influence human causal judgment using findings from \textit{cognitive science} experiments.
To support the evaluation and modeling of human-like causal reasoning in LLMs, we introduce \ourbench, a challenging benchmark with 1,093 instances and detailed human causal reasoning step annotations, and \ourmodel, an approach that systematically integrates theories of token causality and cognitive science into LLMs for human-like causal reasoning.
Results suggest that \ourmodel consistently improves LLMs' causal alignment with humans, and that theory-guided reasoning is highly effective.
\end{abstract}

\section{Introduction}\label{sec:introduction}

\begin{figure*}[t!]
    \centering
    \includegraphics[width=\linewidth]{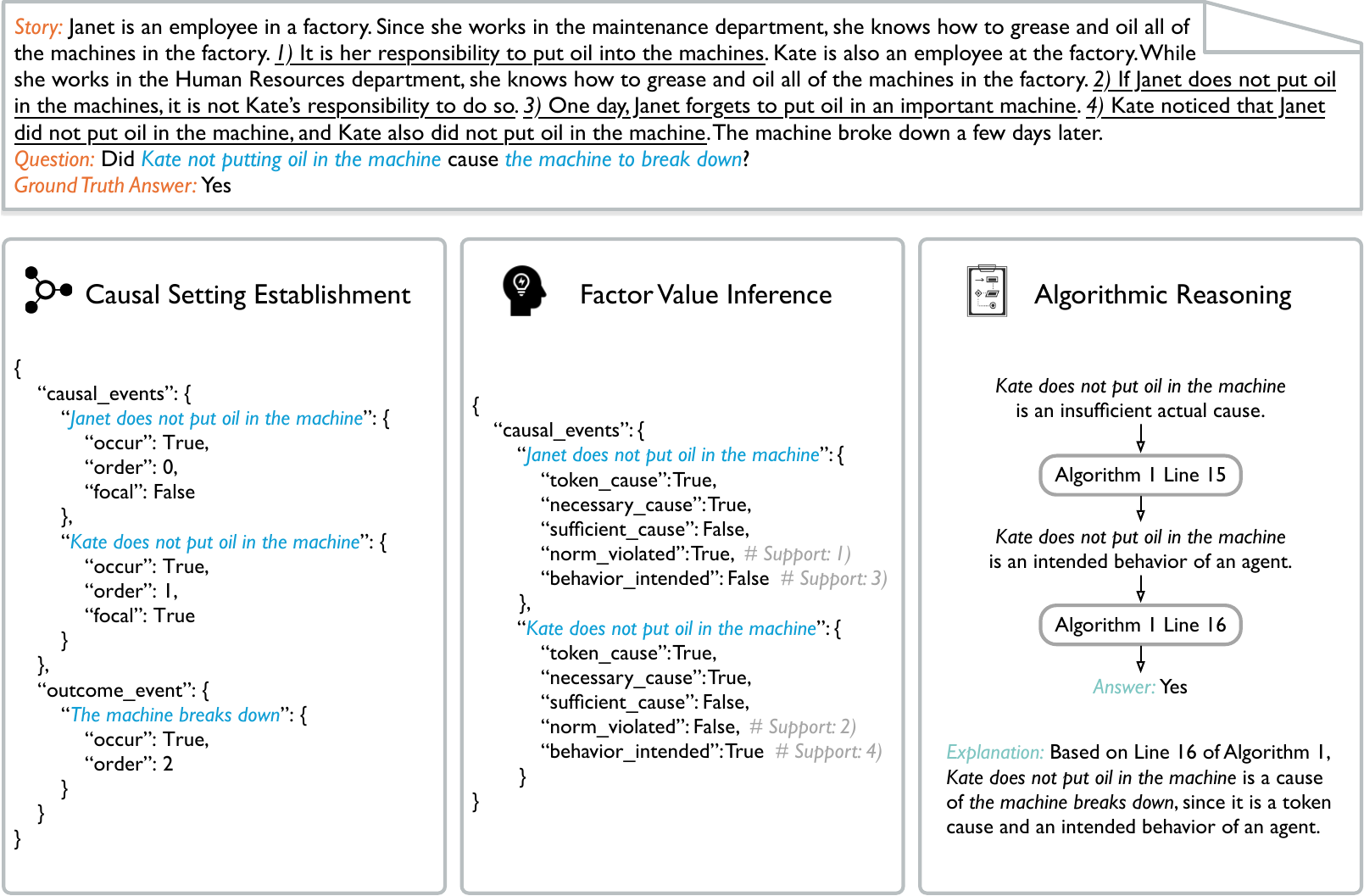}
    \caption{An example in \ourbench. To simulate human causal reasoning, one needs to consider both causal and cognitive factors.}
    \label{fig:benchmark_example}
\end{figure*}

Against the backdrop of the rapid development of LLMs, causal reasoning is increasingly regarded as a core capability for achieving general artificial intelligence, essential for prediction and decision-making. 
However, although current models have demonstrated strong performance on causal benchmarks \cite{roemmele2011choice,du2022care,ashwani2024cause}, they continue to face challenges in achieving \textit{cognitive alignment} with humans. 
Research has shown that models often rely heavily on shallow causal knowledge encoded in their parameters at inference time (\textit{level-1} reasoning), rather than developing a deep understanding of causal mechanisms \cite{zevceviccausal}. 
\citet{chi2024unveiling} have found that there is a fundamental discrepancy between LLMs' reasoning processes and human causal cognition (\textit{level-2} reasoning).

In everyday scenarios, humans more frequently deal with \textit{specific} causal relations between events, e.g., ``Did David's 30 years of smoking habit cause him to get lung cancer last year?'' 
To answer this, they typically draw on \textit{general} causal knowledge, e.g., smoking causes cancer, by understanding the causal chain of events. 
However, such knowledge is insufficient to determine the final judgment, since humans are also influenced by cognitive modulatory factors such as morality, normality, and intent while making judgments in the concrete context \cite{lagnado2013causal,sloman2015causality}.
Humans' causal judgment process corresponds to two complementary research areas: 1) \textit{token causality}, which aims to formalize and axiomatize the intuitive form of causal reasoning used by humans \cite{halpern2016actual}; and 2) \textit{causal judgment}, a branch of cognitive science that investigates how cognitive modulatory factors influence humans' judgment of causes \cite{sloman2015causality}.

However, existing text-based causal reasoning benchmarks largely focus on \textit{general} causal mechanisms and are insufficient for evaluating the \textit{specific} causal relations that humans typically reason about \cite{jin2023cladder,jin2024can,liu2024llms,roemmele2011choice,du2022care,bondarenko2022causalqa}. 
Even among the few benchmarks that address human causal reasoning, their labels are not sufficiently validated by humans \cite{chi2024unveiling}, or they lack complete annotations of human reasoning processes \cite{srivastava2023beyond,suzgun2023challenging}. 
Such a lack of consensus and transparency makes it difficult to determine whether model reasoning steps are genuinely aligned with human cognition. 
Furthermore, existing LLM-based causal reasoning methods have not systematically integrated token causality with cognitive science, and most continue to focus on reasoning about \textit{general} causal mechanisms \cite{jin2023cladder,jiang2024llmcausal,lu2023neurosymbolic}. Although \citet{chi2024unveiling} have begun to consider \textit{specific} causal relations, their methodological designs remain largely at the level of \textit{general} reasoning.

Therefore, we introduce \ourbench, a novel human-like reasoning benchmark containing 1,093 carefully annotated instances with detailed human reasoning step annotations, enabling fine-grained evaluation. 
\ourbench is built upon the causal judgment subset of Big-Bench Hard \cite{suzgun2023challenging}, which provides human consensus labels derived from rigorous human experiments but is too small and lacks human reasoning step annotations.
To support human-like reasoning in LLMs, we further introduce \ourmodel, a novel framework that integrates token causality and cognitive science for human-like causal reasoning.
It identifies candidate causes and outcomes, infers the values of token causality and causal judgment factors for candidate causes, and utilizes these inferred factor values to algorithmically derive the final causal judgment and provide an explanation.

Our contributions are threefold:
1) We introduce \ourbench, the first human-like causal reasoning benchmark with detailed human reasoning step annotations.
2) We propose \ourmodel, the first method that integrates LLMs with token causality and cognitive science theories for human-like causal reasoning.
3) Results suggest that \ourbench is challenging, and that \ourmodel consistently improves performance via the integration of theories into LLMs.

\section{Preliminaries}\label{sec:preliminaries}

\subsection{Type Causality vs. Token Causality}\label{sec:causality_types}

\textit{Type causality} focuses on inferring causal relations between variables, such as causal effect estimation \cite{peters2017elements}. 
In contrast, \textit{token causality} concerns the extent to which particular events bring about other specific events \cite{halpern2016actual,hausman2005causal}.
For example, the question ``Does smoking cause lung cancer?'' pertains to type causality. 
Researchers have traditionally focused more on type causality questions, as they help develop theories and make predictions.
In contrast, the question ``Was Fred's smoking habit responsible for his lung cancer?'' pertains to token causality. 
Its conclusion does not need to generalize to other settings, but is crucial for specific situations.

\subsection{Token Causality}\label{sec:token_causality}
Token causality aims to formalize and axiomatize the intuitive form of causal reasoning used by humans \cite{halpern2016actual}.
The Halpern-Pearl (HP) definition of causality is one of the formal definitions of a ``token cause'' in \textit{specific} causal scenarios \cite{halpern2016actual}.
Here, we first introduce several key concepts and then present the HP definition.

\paragraph{Causal Model.}
A \textit{causal model} $M$ is defined as a pair $(\mathcal{S},\mathcal{F})$, where $\mathcal{S}=(\mathcal{U},\mathcal{V},\mathcal{R})$ is a \textit{signature} consisting of a set of \textit{exogenous variables} $\mathcal{U}$, a set of \textit{endogenous variables} $\mathcal{V}$, and the \textit{ranges} of endogenous variables. $\mathcal{F}$ contains the \textit{structural equations} that determine the values of endogenous variables. Each variable $V\in\mathcal{V}$ takes values from $\mathcal{R}(V)$. $\mathcal{R}(V)=\{0,1\}$ for binary variables.

\paragraph{Primitive Event.}
A \textit{primitive event} have the form $X=x$ for an endogenous variable $X$ and some possible value $x$ of $X$.

\paragraph{Causal Formula.}
A \textit{causal formula} $\psi$ is written as $[\bm{Y}\leftarrow\bm{y}]\varphi$ ($\bm{Y}\subset\mathcal{V}$), meaning that $\varphi$ would hold if an intervention set the variables in $\bm{Y}$ to $\bm{y}$. 

\paragraph{Causal Setting.}
A \textit{causal setting} is a pair $(M,\bm{u})$, where $M$ is the causal model and $\bm{u}$ is the \textit{context} (i.e., an assignment to $\mathcal{U}$). If a formula $\psi$ holds in the causal setting $(M,\bm{u})$, we write $(M,\bm{u})\models\psi$.

\begin{definition}[Token Cause]\label{def:token_cause}
    $\bm{X}=\bm{x}$ is a token cause of $\varphi$ in the causal setting $(M,\bm{u})$ if the following three conditions hold:
    \begin{enumerate}[label=AC\arabic*., leftmargin=*]
        \item $(M,\bm{u})\models (\bm{X}=\bm{x})$ and $(M,\bm{u})\models\varphi$.
        \item There exists a set $\bm{W}$ of variables in $\mathcal{V}$ and a setting $\bm{x}'$ of variables in $\bm{X}$ such that if $(M,\bm{u})\models(\bm{W}=\bm{w}^*)$, then $(M,\bm{u})\models[\bm{X}\leftarrow\bm{x}',\bm{W}\leftarrow\bm{w}^*]\neg\varphi.$
        \item $\bm{X}$ is minimal; there is no strict subset $\bm{X}'$ of $\bm{X}$ such that $\bm{X}'=\bm{x}'$ satisfies conditions AC1 and AC2, where $\bm{x}'$ is the restriction of $\bm{x}$ to the variables in $\bm{X}$.
    \end{enumerate}
\end{definition}

In Def. \ref{def:token_cause}, token causes must be conjunctions of primitive events (i.e., $\bm{X}=\bm{x}$), whereas effects can be any boolean combination of primitive events.
The core clause is AC2, which establishes the counterfactual dependence of the effect $\varphi$ on a candidate cause $\bm{X}=\bm{x}$ by fixing certain variables (in $\bm{W}$) at their actual values. This definition extends the standard \textit{but-for test} (Def. \ref{def:necessary_cause}), allowing it to handle complex real-world scenarios such as \textit{preemption} (see \S\ref{sec:a_failure_case_of_but_for_causation}).

\subsection{Causal Judgment}\label{sec:causal_judgment}
Causal judgment investigates how cognitive modulatory factors influence humans' judgment of causes \cite{sloman2015causality}.
\citet{nie2023moca} have studied and summarized these factors.

\paragraph{Causal Structure.}
Causal structures can be \textit{conjunctive} or \textit{disjunctive}.
A conjunctive causal structure implies that the outcome occurs only if all contributing events jointly occur, while a disjunctive structure indicates that any single event is sufficient to bring about the outcome.
This factor plays a critical role in causal judgment \citep{wells1989mental,mandel2003judgment,sloman2015causality}.

\paragraph{Normality.}
Norms can be \textit{descriptive} or \textit{prescriptive} \citep{halpern2016actual}. Descriptive norms are statistical typicals, such as the mode, mean, or values close to them. Prescriptive norms are moral standards, laws, institutional policies, or even proper functioning in machines or organisms \citep{halpern2016actual}. 
Humans tend to ascribe more causality to abnormal events than to normal ones \citep{knobe2008causal,hitchcock2009cause}.

\paragraph{Intentionality.}
The intent of an agent is considered part of their epistemic state, which typically concerns the agent's awareness when performing specific actions.
An agent who is aware of the potential consequences of their action and deliberately performs an action that leads to a foreseeable bad outcome is judged more harshly than one who lacks the relevant knowledge \citep{samland2016role,kominsky2019immoral}.

\paragraph{Action or Omission.}
People tend to identify actions, rather than omissions, as causes. This phenomenon is known as ``omission bias'' \citep{ritov1992status,baron2004omission}. 
In scenarios whether one agent acts while another does not, the one who acts is more likely to be judged as a cause \citep{henne2017cause,clarke2015causation,descioli2011omission,gerstenberg2021counterfactual}.

\paragraph{Temporal Effect.}
Causal judgment is influenced by the temporal order of events \citep{gerstenberg2012contributions}. When multiple events unfold over time and lead to an outcome, people tend to identify later events as the cause \citep{reuter2014good}. However, this preference depends on how the events causally relate to one another \citep{hilton2010selecting,spellman1997crediting}. When earlier events determine the course of subsequent actions, they are more likely to be selected as causes \citep{henne2021counterfactual}.
\section{Related Work}

\paragraph{Commonsense Causal Reasoning.}
Early causal reasoning studies in the natural language processing domain mainly focused on \textit{commonsense causal knowledge} \cite{roemmele2011choice,sap2019social,zheng2023preserving,zhang2022rock}. However, they often treat language models as ``knowledge bases'', resulting in shallow causal reasoning \cite{jiang2020can,petroni2019language}.

\paragraph{Formal Causal Reasoning.}
Recent studies have begun to focus on \textit{formal causal reasoning}, requiring systems to make decisions based on formal causal frameworks such as structural causal models \cite{jin2023cladder,jin2024can,liu2024llms}.
For example, CLadder \cite{jin2023cladder} evaluates whether LLMs can perform causal inference according to a set of well-defined causal rules; Corr2Cause \cite{jin2024can} requires LLMs to determine causal relations between variables based on a set of correlational statements; and QRData \cite{liu2024llms} evaluates LLMs' causal reasoning abilities when processing real-world data. Also, CausalCoT \cite{jin2023cladder} provides a chain-of-thought (CoT) baseline for step-by-step formal causal reasoning.
However, these studies overlook non-formal causal judgment factors that influence humans' judgment of causes \cite{sloman2015causality}.

\paragraph{Human-like Causal Reasoning.}
No prior work has systematically combined token causality and cognitive science for human-like causal reasoning.
On the data side, type causality benchmarks struggle to model human judgments of complex causal links between events in concrete situations \cite{jin2023cladder,jin2024can,liu2024llms,roemmele2011choice,du2022care,bondarenko2022causalqa}. Recently, \citet{chi2024unveiling} proposed the task of \textit{human-like causal reasoning} and introduced the CausalProbe-2024 benchmark. However, its labels are fully generated by LLMs, and it does not provide human reasoning step annotations. Other studies leverage findings from human experiments to derive human consensus labels \cite{srivastava2023beyond,suzgun2023challenging,nie2023moca}, but they do not provide complete reasoning step annotations.
On the methodological side, \citet{chi2024unveiling} proposed G$^2$-Reasoner, which incorporates general knowledge and goal-oriented prompts into LLMs' causal reasoning processes. However, this method still focuses on type causality and does not integrate human causal cognition.

\section{The \ourbench Benchmark}\label{sec:hcr_bench}
We construct \ourbench to address the lack of consensus and transparency in existing benchmarks by reusing human consensus labels validated through rigorous human experiments and providing fine-grained annotations of human causal reasoning processes.
As shown in Fig. \ref{fig:benchmark_example}, \ourbench is defined as $\mathcal{D}:=\{t_i,q_i,a_i,r_i\}_{i=1}^N$, where each sample comprises a causal story $t_i$, a causal query $q_i$, a binary answer $a_i\in\{\mathrm{Yes},\mathrm{No}\}$, and the reasoning process $r_i$. The goal of \ourbench is to evaluate whether an LLM can correctly map $q$ to $a$ and produce a correct reasoning process $r$.

\begin{table}[h!]
    \centering
    \caption{Comparison of causal reasoning benchmarks. Type: type vs. token causality? Valid. Labels: labels human-validated? Caus.: causal structures annotated? Cogn.: cognitive factors annotated?}
    \resizebox{\linewidth}{!}{
    \begin{tabular}{lcccc}
        \toprule
        \textbf{Dataset}
        & \textbf{\begin{tabular}{@{}c@{}}Type\end{tabular}}
        & \textbf{\begin{tabular}{@{}c@{}}Valid.\\Labels\end{tabular}}
        & \textbf{\begin{tabular}{@{}c@{}}Caus.\end{tabular}}
        & \textbf{\begin{tabular}{@{}c@{}}Cogn.\end{tabular}} \\
        \midrule
        COPA \cite{roemmele2011choice} & Type & \textcolor{green}{\ding{51}} & \textcolor{red}{\ding{55}} & \textcolor{red}{\ding{55}} \\
        e-CARE \cite{du2022care} & Type & \textcolor{green}{\ding{51}} & \textcolor{red}{\ding{55}} & \textcolor{red}{\ding{55}} \\
        CLadder \cite{jin2023cladder} & Type & \textcolor{red}{\ding{55}} & \textcolor{green}{\ding{51}} & \textcolor{red}{\ding{55}} \\
        Corr2Cause \cite{jin2024can} & Type & \textcolor{red}{\ding{55}} & \textcolor{green}{\ding{51}} & \textcolor{red}{\ding{55}} \\
        QRData \cite{liu2024llms} & Type & \textcolor{green}{\ding{51}} & \textcolor{red}{\ding{55}} & \textcolor{red}{\ding{55}} \\
        Big-Bench \cite{srivastava2023beyond} & Token & \textcolor{green}{\ding{51}} & \textcolor{red}{\ding{55}} & \textcolor{red}{\ding{55}} \\
        Big-Bench Hard \cite{suzgun2023challenging} & Token & \textcolor{green}{\ding{51}} & \textcolor{red}{\ding{55}} & \textcolor{red}{\ding{55}} \\
        MoCa \cite{nie2023moca} & Token & \textcolor{green}{\ding{51}} & \textcolor{red}{\ding{55}} & \textcolor{green}{\ding{51}} \\
        CausalProbe-2024 \cite{chi2024unveiling} & Token & \textcolor{red}{\ding{55}} & \textcolor{red}{\ding{55}} & \textcolor{red}{\ding{55}} \\
        \midrule
        \ourbench (Ours) & Token & \textcolor{green}{\ding{51}} & \textcolor{green}{\ding{51}} & \textcolor{green}{\ding{51}} \\
        \bottomrule
    \end{tabular}
    }
    \label{tab:benchmark_comparison}
\end{table}

\subsection{Dataset Construction}\label{sec:dataset_construction}
This section introduces the construction process of \ourbench, including data cleaning, data annotation, data augmentation, and quality control.

\paragraph{Data Source.}
The causal judgment task of Big-Bench \cite{srivastava2023beyond} tests whether LLMs can read a short story with multiple cause-effect events and answer causal queries in the same manner as humans would. It contains 190 stories from 30 papers, each of which conducted rigorous human experiments to derive human consensus labels.
Big-Bench Hard \cite{suzgun2023challenging} selects 23 challenging tasks from Big-Bench, with causal judgment being one of them. Its causal judgment subset contains 187 stories, including 141 \textit{causation} queries (e.g., ``Did the red wire cause the machine to short circuit?'') and 46 \textit{intention} queries (e.g., ``Did the man shoot the dear intentionally?'').
\ourbench derives from this subset.

\paragraph{Data Cleaning.}
Before annotation, we manually remove 54 samples that are duplicate, lack critical background information, or focus on intention, and correct 2 partially flawed samples. This process reduces noise and preserves only samples of interest, resulting in a curated set of 133 samples. 

\paragraph{Data Annotation.}
We manually annotate the reasoning steps for the 133 samples. Each annotation is a unified JSON combining the two JSONs shown in Fig. \ref{fig:benchmark_example}.
Specifically, for all events, we annotate the occurrences $f_{\mathrm{occur}}$ and temporal orders $f_{\mathrm{order}}$, with $f_{\mathrm{order}}$ being an integer-valued factor indicating an event's relative sequential position starting from 0. These factors model action or omission and temporal effects. For causal events, we additionally annotate three types of binary factors: 
1) \textit{auxiliary factor}: $f_{\mathrm{focal}}$, which records whether the event is queried; 
2) \textit{causal factors}: token cause $f_{\mathrm{tc}}$ (Def. \ref{def:token_cause}), necessary cause $f_{\mathrm{nc}}$ (Def. \ref{def:necessary_cause}), and sufficient cause $f_{\mathrm{sc}}$ (Def. \ref{def:sufficient_cause}), which represents whether the event is a token cause, necessary cause, or sufficient cause of the outcome; and
3) \textit{cognitive factors}: normality $f_{\mathrm{n}}$ and intentionality $f_{\mathrm{i}}$, where $f_{\mathrm{n}}$ determines whether the event involves norm violations, and $f_{\mathrm{i}}$ determines whether the event is an agent's action, whether the agent is aware of its potential outcomes, and whether the agent knowingly performs an action that leads to a foreseeable bad outcome.

\paragraph{Data Augmentation.}
We expand the dataset with a two-stage augmentation process to reach a reasonable evaluation scale.
First, we manually create 30 new samples from existing stories by altering the query to focus on causal events that are not queried, i.e., ``Did $X'$ cause $Y$?'' In such cases, only $f_{\mathrm{focal}}$ and $a$ need to be updated.
Second, we employ GPT-4o \citep{hurst2024gpt} to synthesize more samples by rewriting the stories of the existing 163 samples (\textit{seed samples}).
This is a three-stage process.
Based on a seed sample, it first rewrites the story into a completely new real-world scenario while preserving the reasoning logic. That is, only the descriptions of events change, while their factor values remain unchanged. Then, based on the new events, it generates a new query. The gold answer remains unchanged.
Finally, it adds or removes details from the new story and reorganizes its structure to make it more distinctive while preserving the reasoning logic.
We generate 6 samples for 155 seed samples, obtaining 930 generated samples. Together with the 163 seed samples, the dataset has 1,093 samples.
Generated samples are generally more challenging since they contain more spurious correlations and fewer explicit causal cues (\S\ref{sec:case_study}), e.g., ``When $E_1$ and $E_2$ occur, $O$ will occur.''
Spurious correlations are common in real-world scenarios, and humans often identify causal relations without explicitly stating such cues \cite{matute2015illusions,vasilyeva2015explanations}.

\paragraph{Quality Control.}
We apply automated and manual methods to verify the quality of \ourbench.
For seed samples, we manually review each sample and iteratively refine our annotations before generation, based on the metadata and annotations provided by \citet{suzgun2023challenging} and \citet{nie2023moca}, as well as LLM-generated pseudo-reasoning steps.
For generated samples, we implement code to automatically check whether the reasoning logic of each generated sample is consistent with that of its seed sample. Fewer than 1\% of samples are found to be inconsistent (mostly involving incorrect values for $f_{\mathrm{i}}$), and these are manually corrected.
We also verify whether the reasoning logic of each generated sample corresponds to the correct gold answer using Alg. \ref{alg:theory_guided_reasoning}; all samples pass this automatic check.
In addition, we randomly sample 310 generated samples and evaluate whether 1) the story contains sufficient information to answer the query, 2) the reasoning logic is consistent with the story, and 3) the gold answer is correct. Only 2 samples (0.6\%) lack essential reasoning information; we regenerate and verify them.

\subsection{Dataset Statistics}\label{sec:dataset_statistics}
\ourbench contains 1,093 samples, with a relatively balanced proportion of positive and negative samples. Most queries concern a single causal event, while a small number involve multiple causal events.

\begin{table}[h!]
\centering
\caption{\ourbench statistics.}
\vspace{1.5pt}
\resizebox{\columnwidth}{!}{
\begin{tabular}{@{}lll@{}}
\toprule
\multicolumn{1}{l}{\textbf{Category}} & \multicolumn{1}{l}{\textbf{Statistic}} & \multicolumn{1}{c}{\textbf{Total}} \\ \midrule
\multicolumn{1}{l}{\textsc{Size}}      & \# Samples          & \multicolumn{1}{c}{1,093} \\
\multicolumn{1}{l}{\textsc{Story}}     & \# Sentences (per sample) & \multicolumn{1}{c}{8.49} \\
\multicolumn{1}{c}{}          & \# Words (per sample)     & \multicolumn{1}{c}{162.70} \\
\multicolumn{1}{l}{\textsc{Query}}     & \# Queried causal events (per sample)    & \multicolumn{1}{c}{1.02} \\
\multicolumn{1}{l}{\textsc{Reasoning}} & \# Events (per sample)    & \multicolumn{1}{c}{3.39} \\
\multicolumn{1}{c}{}          & \# Causal events (per sample)    & \multicolumn{1}{c}{2.39} \\
\multicolumn{1}{c}{}          & \# Outcome events (per sample)  & \multicolumn{1}{c}{1} \\
\multicolumn{1}{l}{\textsc{Answer}}    & Positive class      & \multicolumn{1}{c}{53.8\%} \\
& Negative class      & \multicolumn{1}{c}{46.2\%} \\
\bottomrule
\end{tabular}
}
\label{tab:benchmark_statistics}
\end{table}
\section{The \ourmodel Framework}\label{sec:hcr_reasoner}

\begin{figure*}[t!]
    \centering
    \includegraphics[width=\linewidth]{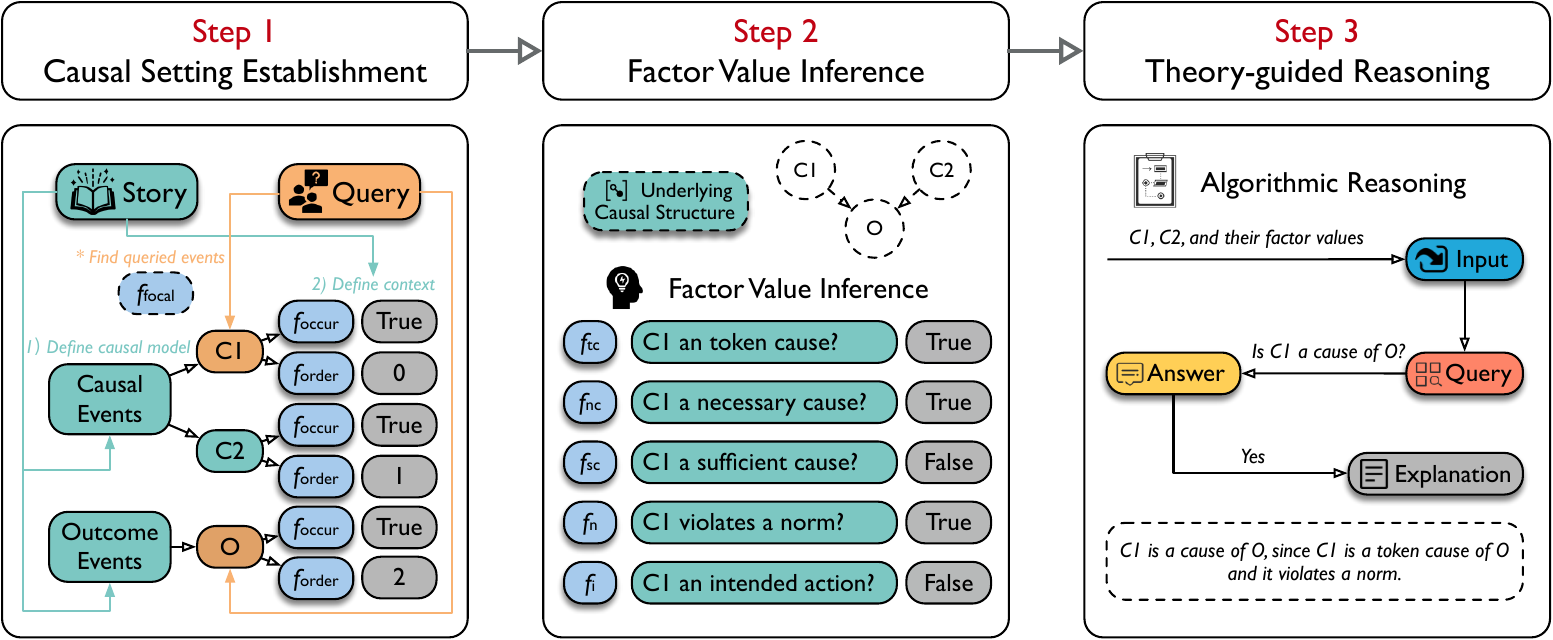}
    \caption{The overview of \ourmodel.}
    \label{fig:framework_overview}
\end{figure*}

We propose \ourmodel to integrate token causality and cognitive science theories into LLMs for step-by-step, explainable human-like causal reasoning.
As shown in Fig. \ref{fig:framework_overview}, it operates in three stages, ending with an answer and explanation.

\subsection{Causal Setting Establishment}\label{sec:causal-setting-establishment}
\ourmodel first sets up the causal setting, including the causal model $M$ and the context $\bm{u}$.

\paragraph{Causal Model.}
In \ourmodel, $\mathcal{V}$ is partitioned into a causal event set $\mathcal{C}$ and an outcome event set $\mathcal{O}$, with $\mathcal{R}(V)=\{0,1\}$.
To set up the causal model $M$, \ourmodel extracts the events from the story $t$ and query $q$, and combines the outcome events into a unified outcome event $O$\footnote{The instruction states and generally ensures that causal events are minimal and do not overlap (Def. \ref{def:token_cause}, AC3).}.
It also records $f_{\mathrm{focal}}$ for each causal event.

\paragraph{Context.}
The context $\bm{u}$ represents the values of exogenous variables, which in turn determine the values of endogenous variables via the structural equations $\mathcal{F}$. In practice, \ourmodel extracts the values of endogenous variables, i.e., the values of $f_{\mathrm{occur}}$, as a proxy.
It also records the temporal order $f_{\mathrm{order}}$ for each event.
The stage of causal setting establishment is formalized as:
\begin{equation*}
\left<\mathcal{C},O,f_{\mathrm{focal}},f_{\mathrm{occur}},f_{\mathrm{order}}\right>\leftarrow\Phi(t,q,I_{\mathrm{m}},I_{\mathrm{c}}),
\end{equation*}
where $\Phi$ is an LLM. $I_{\mathrm{m}}$ and $I_{\mathrm{c}}$ are instructions for $\Phi$ to extract $M$ and $\bm{u}$, respectively.


\subsection{Factor Value Inference}\label{sec:factor_value_inference}
\ourmodel then infers the values of causal and cognitive factors for each causal event.
The stage of factor value inference is formalized as:
\begin{equation*}
\left<f_{\mathrm{tc}},f_{\mathrm{nc}},f_{\mathrm{sc}},f_{\mathrm{n}},f_{\mathrm{i}}\right>=\Phi(t,I_{\mathrm{f}},\mathcal{C},O,f_{\mathrm{occur}}),
\end{equation*}
where $I_{\mathrm{f}}$ is the instruction for $\Phi$ to infer factor values.
$f_{\mathrm{occur}}$ is used for counterfactual reasoning to determine the values of the causal factors.

\subsection{Theory-guided Reasoning}\label{sec:theory_guided_reasoning}

\begin{algorithm}[h!]
\small
\caption{\small{Theory-guided Reasoning}}
\label{alg:theory_guided_reasoning}
\KwIn{
Causal event set $\mathcal{C}$ and their values of factors $f_{\mathrm{f_{focal}}}$, $f_{\mathrm{occur}}$, $f_{\mathrm{order}}$, $f_{\mathrm{tc}}$, $f_{\mathrm{nc}}$, $f_{\mathrm{sc}}$, $f_{\mathrm{n}}$, and $f_{\mathrm{i}}$.
}
\KwOut{A binary answer and an explanation.}
\ForEach{$E \in \mathcal{C}$}{
    \textcolor{gray}{\lIf{$\lnot f_{\mathrm{focal}}^E$}{CONTINUE}}
    \lIf{$f_{\mathrm{nc}}^E \land f_{\mathrm{sc}}^E$}{\Return ``Yes''}
    \ElseIf{$\lnot f_{\mathrm{nc}}^E \land f_{\mathrm{sc}}^E$}{
        $\mathcal{C}_s \leftarrow \{E' \in \mathcal{C} \mid \lnot f_{\mathrm{nc}}^{E'} \land f_{\mathrm{sc}}^{E'}\}$\;
        \If{$\mathrm{len}(\mathrm{set}(\{f_{\mathrm{order}}^{E_s}\mid E_s\in\mathcal{C}_s\}))\neq 1$}{
            \lIf{$f_{\mathrm{order}}^E = \max\{f_{\mathrm{order}}^{E_s} \mid E_s \in \mathcal{C}_s\}$}{\Return ``Yes''}
        }
        \Else{
            \lForEach{$E_s \in \mathcal{C}_s$}{
                $f_{\mathrm{resp}}^{E_s} \leftarrow \Phi(f_{\mathrm{n}}^{E_s}, f_{\mathrm{i}}^{E_s})$
            }
            \lIf{$f_{\mathrm{resp}}^E \leftarrow \max\{f_{\mathrm{resp}}^{E_s} \mid E_s \in \mathcal{C}_s\}$}{\Return ``Yes''}
        }
    }
    \ElseIf{$f_{\mathrm{tc}}^E \land \lnot f_{\mathrm{sc}}^E$}{
        \lIf{$f_{\mathrm{n}}^E \lor f_{\mathrm{i}}^E$}{\appendYes}
        \Else{
            $\mathcal{C}_a \leftarrow \{E' \in \mathcal{C} \mid f_{\mathrm{tc}}^{E'} \land \lnot f_{\mathrm{sc}}^{E'}\}$\;
            \lForEach{$E_a\in\mathcal{C}_a$}{$f_{\mathrm{resp}}^{E_a}=\Phi(f_{\mathrm{order}}^{E_a})$}
            \lIf{$\forall E_a\in\mathcal{C}_a\setminus \{E\},f_{resp}^E>f_{resp}^{E_a}$}{\Return ``Yes''}
        }
    }
    \textcolor{gray}{\lElse{CONTINUE}}
}
\Return ``No''
\end{algorithm}

Based on the extracted causal knowledge, Alg. \ref{alg:theory_guided_reasoning} conducts theory-guided human-like causal reasoning for binary queries such as ``Did $X$ cause $Y$?''.
For each queried causal event $E$, the algorithm first partitions the reasoning process using $f_{\mathrm{tc}}^E$, $f_{\mathrm{nc}}^E$, and $f_{\mathrm{sc}}^E$ (see \S\ref{sec:reasoning_process_partitioning}). Then, for each partition, it determines whether $E$ is a judged cause of $O$ (see \S\ref{sec:rules_within_each_partition}). Finally, it outputs a binary answer and provides an explanation (see \S\ref{sec:details_of_explanation_generation}).
In Alg. \ref{alg:theory_guided_reasoning}, we further introduce a responsibility factor $f_{\mathrm{resp}}$ to measure relative responsibility. $f_{\mathrm{resp}}^E$ is defined as the relative causal contribution of $E$ to $O$ compared with those of other contributing causes \cite{kiciman2023causal}.

Specifically, 
1) If $E$ is a sufficient token cause, it is a cause of $O$ since it is complete.
2) If $E$ is not a token cause, it is not a cause of $O$ since it is not in the causal chain.
3) If $E$ is one among multiple disjunctive contributing causes, Alg. \ref{alg:theory_guided_reasoning} first examines their relative temporal orders to assess \textit{preemption} (see \S\ref{sec:a_failure_case_of_but_for_causation}). If they occur simultaneously, their $f_{\mathrm{resp}}$ values based on $f_{\mathrm{n}}$ and $f_{\mathrm{i}}$ are compared to determine the cause: $E$ is a cause of $O$ only if it takes the largest responsibility. Otherwise, if $E$ occurs the earliest, it is a cause of $O$; otherwise, it is not.
4) If $E$ is part of the conjunctive set of contributing causes, Alg. \ref{alg:theory_guided_reasoning} first evaluates whether $f_{\mathrm{n}}$ or $f_{\mathrm{i}}$ enhances $E$'s causal strength. If so, $E$ is a cause of $O$. Otherwise, their $f_{\mathrm{resp}}$ values based on $f_{\mathrm{order}}$ are compared to determine the cause: $E$ is a cause of $O$ only if it takes the unique largest responsibility.
Finally, Alg. \ref{alg:theory_guided_reasoning} provides an explanation for each answer.

\section{Experiments}\label{sec:experiments}

\subsection{Setups and Baselines}\label{sec:setups_and_baselines}

\paragraph{Setups.}
We implement \ourmodel and the baselines on recent LLMs for comparison, including Qwen2.5-32B/72B-Instruct \citep{qwen2025qwen25technicalreport}, DeepSeek-V3 \citep{liu2024deepseek}, Gemini-2.0-Flash \citep{gemini}, Claude-3.5-Sonnet \citep{claude}, GPT-4o-2024-11-20 \citep{hurst2024gpt}, and GPT-4-0613 \citep{achiam2023gpt}.
\ourmodel is also tested on reasoning models including DeepSeek-R1 \citep{guo2025deepseek} and QwQ-32B \citep{qwq}.
Temperature is set to 0 for reproducibility.
We report averaged accuracy over 10 runs for the pilot study and a single run for the main results due to high API costs.

\paragraph{Baselines.}
We compare \ourmodel with three baselines: 
1) \textit{vanilla}, which directly prompts an LLM to output a Yes/No answer given the story and query; 
2) \textit{zero-shot CoT}, which appends the instruction ``Let's think step by step.'' to the vanilla prompt; and
3) \textit{manual CoT} \citep{chen2024causal}, which replaces ``Let's think step by step.'' with three in-context demonstrations containing manually written reasoning steps.
GPT-4 + manual CoT \citep{chen2024causal} is the previous state-of-the-art on Big-Bench Hard causal judgment.

\subsection{Pilot Study}\label{sec:pilot_study}

\begin{table}[t]
    \centering
    \caption{Results of the pilot study. Human average is directly cited from \citet{suzgun2023challenging}. See \S\ref{sec:pilot_results} for complete results and significance testing results.}
    \resizebox{\columnwidth}{!}{
    \begin{tabular}{@{}lccc@{}}
    \toprule
    \multicolumn{1}{l}{\textbf{Methods}} & \multicolumn{1}{c}{\textbf{Acc. (C.)}} & \multicolumn{1}{c}{\textbf{Acc. (I.)}} & \multicolumn{1}{c}{\textbf{Acc.}} \\ \midrule
    
    \multicolumn{1}{l}{Human Average} & - & - & \multicolumn{1}{c}{69.60\%} \\ \midrule
    
    \multicolumn{1}{l}{Qwen2.5-32B-Instruct} & 65.25\% & 80.43\% & \multicolumn{1}{c}{68.98\%} \\
    \multicolumn{1}{l}{~~+ \ourmodel} & 70.14\% & 80.43\% & \multicolumn{1}{c}{72.67\%} \\ \midrule
    
    \multicolumn{1}{l}{Qwen2.5-72B-Instruct} & 65.89\% & 77.83\% & \multicolumn{1}{c}{68.82\%} \\
    \multicolumn{1}{l}{~~+ \ourmodel} & 71.63\% & 79.13\% & \multicolumn{1}{c}{73.48\%} \\ \midrule
    
    \multicolumn{1}{l}{DeepSeek-V3} & 66.45\% & 77.61\% & \multicolumn{1}{c}{69.20\%} \\
    \multicolumn{1}{l}{~~+ \ourmodel} & 69.93\% & 78.48\% & \multicolumn{1}{c}{72.03\%} \\ \midrule
    
    \multicolumn{1}{l}{Claude-3.5-Sonnet} & 67.23\% & 73.04\% & \multicolumn{1}{c}{68.66\%} \\
    \multicolumn{1}{l}{~~+ \ourmodel} & 72.62\% & 75.43\% & \multicolumn{1}{c}{73.32\%} \\ \midrule
    
    \multicolumn{1}{l}{GPT-4o-2024-11-20} & 54.89\% & 85.22\% & \multicolumn{1}{c}{62.35\%} \\
    \rowcolor{gray!20}
    \multicolumn{1}{l}{~~+ \ourmodel} & 70.50\% & 85.22\% & \multicolumn{1}{c}{74.12\%} \\ \midrule
    
    \multicolumn{1}{l}{GPT-4-0613} & 62.48\% & 80.22\% & \multicolumn{1}{c}{66.84\%} \\
    \rowcolor{gray!20}
    \multicolumn{1}{l}{~~+ \ourmodel} & \textbf{74.61\%} & 78.26\% & \multicolumn{1}{c}{\textbf{75.51\%}} \\ \bottomrule
    \end{tabular}
    }
    \label{tab:pilot_results}
\end{table}

We preliminarily evaluate \ourmodel on Big-Bench Hard causal judgment to assess its effectiveness (Fig. \ref{tab:pilot_results}). We report both overall accuracy and fine-grained accuracies for two query types: causation (C.) and intention (I.).

\paragraph{Overall Findings.}
1) \textit{\ourmodel consistently and significantly improves the performance of all tested LLMs.}
GPT-4 + \ourmodel achieves an overall accuracy of 75.51\%, followed by GPT-4o at 74.12\%. This shows that integrating theory into LLMs is promising. 
2) \textit{Closed-source LLMs benefit more from \ourmodel.} We attribute this to the stronger capacity of these LLMs to handle complex, multi-step CoT instructions.

\paragraph{Comparison with Humans.}
Comparing with humans is crucial, as it benchmarks the performance of LLMs against the established human consensus on causal judgments. 
In Big-Bench Hard causal judgment, the Yes/No labels are derived from rigorous human experiments, providing a reliable and externally valid testbed \cite{suzgun2023challenging}. 
Although the average human score appears relatively low (69.60\%), this is largely an artifact of the evaluation procedure: the original proportional human responses computed from individual votes are binarized, such that both a 100\%-0\% split and a 51\%-49\% split are collapsed into a single binary label. As a result, it is nearly impossible for an individual to achieve consistently high accuracy, particularly on instances with near-even splits, even though the best-performing human still reaches 100\% accuracy \cite{suzgun2023challenging}. 
Critically, after applying \ourmodel, all LLMs exceed the human average, demonstrating that \textit{\ourmodel can effectively and efficiently enable LLMs to replicate human consensus in causal judgments without relying on domain experts or crowd annotators}.

\subsection{Main Results}\label{sec:main_results}
Main results are presented in Tab. \ref{tab:main_results}. We find that integrating explicit theories into LLMs is effective.

\paragraph{Main Findings.}
1) \textit{\ourmodel drives consistent improvement, with greater gains in stronger models.}
This finding is consistent with our pilot study, highlighting the framework's capacity to effectively utilize the advanced reasoning abilities of state-of-the-art models.
2) \textit{Zero-shot and manual CoT are insufficient for the task.} Neither zero-shot CoT nor manual CoT yields consistent performance improvements, highlighting the inherent difficulty of the task. This finding is consistent with prior work \cite{chen2024causal}.
3) \textit{Theory-grounded reasoning delivers substantial improvements.}
The consistent boost in accuracy across diverse LLMs confirms the core benefit of our approach: grounding reasoning in explicit theories.

\paragraph{Human Evaluation.}
Due to the high cognitive load of the task, we randomly selected 100 samples from \ourbench and conducted a human evaluation. Four participants are involved, achieving accuracies of 60\%, 66\%, 68\%, and 69\%, with an average accuracy of 65.75\%. This average accuracy is lower than that of Big-Bench Hard causal judgment, consistent with our findings in \S\ref{sec:case_study} that \ourbench is more challenging than Big-Bench Hard causal judgment.

\begin{table}[t!]
\centering
\caption{The results of different LLMs on \ourbench. \textbf{CE}$\checkmark$ and \textbf{OE}$\checkmark$ are the proportions of correctly identified causal and outcome events, respectively.}
\resizebox{\columnwidth}{!}{
\begin{tabular}{@{}lcccr@{}}
\toprule
\multicolumn{1}{l}{\textbf{Methods}} & \multicolumn{1}{c}{\textbf{Accuracy}} & \multicolumn{1}{c}{\textbf{CE$\checkmark$}} & \multicolumn{1}{c}{\textbf{OE$\checkmark$}} \\
\midrule

\multicolumn{4}{l}{\textit{Open-source LLMs}} \\
\midrule

\multicolumn{1}{l}{Qwen2.5-32B-Instruct} & 62.67\% & \multicolumn{1}{c}{-} & \multicolumn{1}{c}{-} \\
\multicolumn{1}{l}{~~+ zero-shot CoT} & 61.94\% & \multicolumn{1}{c}{-} & \multicolumn{1}{c}{-} \\
\multicolumn{1}{l}{~~+ manual CoT} & 62.85\% & \multicolumn{1}{c}{-} & \multicolumn{1}{c}{-} \\
\multicolumn{1}{l}{~~+ \ourmodel} & 64.78\% & 95.91\% & 93.96\% \\
\midrule

\multicolumn{1}{l}{Qwen2.5-72B-Instruct} & 64.87\% & \multicolumn{1}{c}{-} & \multicolumn{1}{c}{-} \\
\multicolumn{1}{l}{~~+ zero-shot CoT} & 62.03\% & \multicolumn{1}{c}{-} & \multicolumn{1}{c}{-} \\
\multicolumn{1}{l}{~~+ manual CoT} & 62.31\% & \multicolumn{1}{c}{-} & \multicolumn{1}{c}{-} \\
\multicolumn{1}{l}{~~+ \ourmodel} & 67.52\% & 94.53\% & 91.95\% \\
\midrule

\multicolumn{1}{l}{DeepSeek-V3} & 63.49\% & \multicolumn{1}{c}{-} & \multicolumn{1}{c}{-} \\
\multicolumn{1}{l}{~~+ zero-shot CoT} & 64.68\% & \multicolumn{1}{c}{-} & \multicolumn{1}{c}{-} \\
\multicolumn{1}{l}{~~+ manual CoT} & 63.95\% & \multicolumn{1}{c}{-} & \multicolumn{1}{c}{-} \\
\multicolumn{1}{l}{~~+ \ourmodel} & 67.61\% & 96.06\% & 91.40\% \\
\midrule

\multicolumn{4}{l}{\textit{Closed-source LLMs}} \\
\midrule

\multicolumn{1}{l}{Gemini-2.0-Flash} & 60.20\% & \multicolumn{1}{c}{-} & \multicolumn{1}{c}{-} \\
\multicolumn{1}{l}{~~+ zero-shot CoT} & 58.65\% & \multicolumn{1}{c}{-} & \multicolumn{1}{c}{-} \\
\multicolumn{1}{l}{~~+ manual CoT} & 58.55\% & \multicolumn{1}{c}{-} & \multicolumn{1}{c}{-} \\
\multicolumn{1}{l}{~~+ \ourmodel} & 64.96\% & \textbf{96.63\%} & 95.61\% \\
\midrule

\multicolumn{1}{l}{Claude-3.5-Sonnet} & 63.68\% & \multicolumn{1}{c}{-} & \multicolumn{1}{c}{-} \\
\multicolumn{1}{l}{~~+ zero-shot CoT} & 65.42\% & \multicolumn{1}{c}{-} & \multicolumn{1}{c}{-} \\
\multicolumn{1}{l}{~~+ manual CoT} & 62.67\% & \multicolumn{1}{c}{-} & \multicolumn{1}{c}{-} \\
\rowcolor{gray!20}
\multicolumn{1}{l}{~~+ \ourmodel} & 70.54\% & 95.52\% & 91.67\% \\
\midrule

\multicolumn{1}{l}{GPT-4o-2024-11-20} & 58.65\% & \multicolumn{1}{c}{-} & \multicolumn{1}{c}{-} \\
\multicolumn{1}{l}{~~+ zero-shot CoT} & 59.38\% & \multicolumn{1}{c}{-} & \multicolumn{1}{c}{-} \\
\multicolumn{1}{l}{~~+ manual CoT} & 60.66\% & \multicolumn{1}{c}{-} & \multicolumn{1}{c}{-} \\
\multicolumn{1}{l}{~~+ \ourmodel} & 68.07\% & 94.45\% & 92.04\% \\
\midrule

\multicolumn{1}{l}{GPT-4-0613} & 63.77\% & \multicolumn{1}{c}{-} & \multicolumn{1}{c}{-} \\
\multicolumn{1}{l}{~~+ zero-shot CoT} & 62.49\% & \multicolumn{1}{c}{-} & \multicolumn{1}{c}{-} \\
\multicolumn{1}{l}{~~+ manual CoT} & 66.51\% & \multicolumn{1}{c}{-} & \multicolumn{1}{c}{-} \\
\rowcolor{gray!20}
\multicolumn{1}{l}{~~+ \ourmodel} & \textbf{71.82\%} & 94.26\% & \textbf{96.07\%} \\
\bottomrule
\end{tabular}
}
\label{tab:main_results}
\end{table}

\subsection{Fine-Grained Analysis}\label{sec:fine_grained_analysis}

\paragraph{Fine-Grained Accuracies.}
To better understand the internal mechanism of \ourmodel, we report fine-grained accuracies in Tab. \ref{tab:main_results} and Fig. \ref{fig:fine_grained_results}.
We find that \textit{overall performance is driven by factor value inference}, since all LLMs demonstrate high accuracy in the initial stage of identifying events.
However, we observe no clear correlation between these fine-grained accuracies and overall accuracy.

\begin{figure}[t]
    \centering
    \includegraphics[width=\columnwidth]{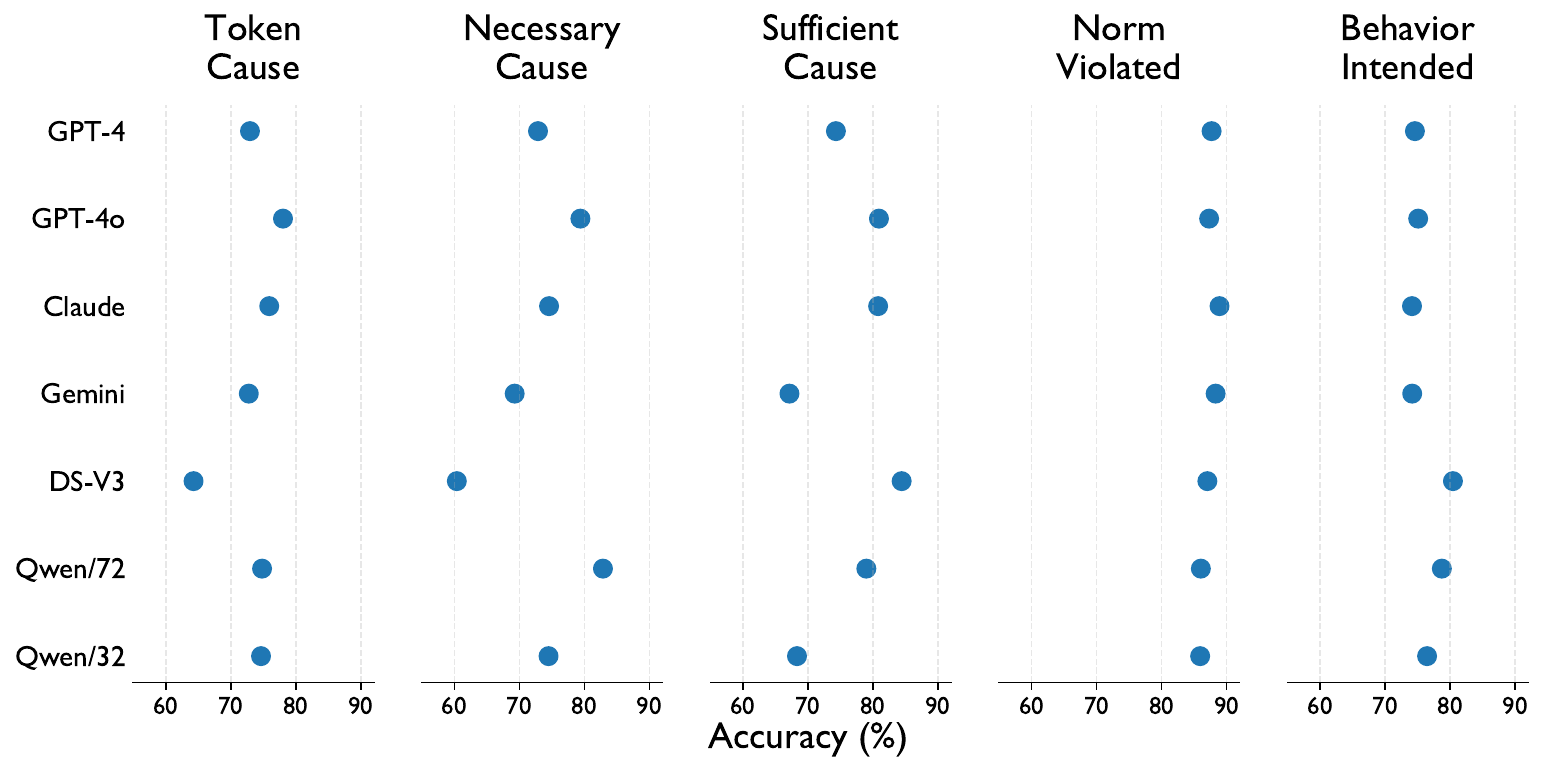}
    \caption{Fine-grained accuracies of different factors.}
    \label{fig:fine_grained_results}
\end{figure}

\paragraph{Causal Analysis.}
To examine the extent to which different factors influence overall accuracy, we perform a causal analysis.
In Fig. \ref{fig:causal_analysis_results}, we plot the averaged average treatment effect (ATE) of each factor on overall accuracy across two estimation methods: ordinary least squares (OLS) \cite{angrist2009mostly} and propensity score matching (Matching) \cite{rosenbaum1983central}, marking factors with $p\geq 0.05$ with cross marks.
Results suggest that \textit{only Qwen2.5-72B-Instruct and Claude-3.5-Sonnet exhibit faithful reasoning, while GPT-4 appears to utilize shortcuts.} 
Ideally, all factors except $f_{\mathrm{nc}}$, which is primarily used for partitioning the reasoning process and does not contribute to determining causes, should causally contribute to overall accuracy. However, only two models align well with this expectation. GPT-4 appears to rely heavily on $f_{\mathrm{nc}}$ while underutilizing $f_{\mathrm{tc}}$, suggesting that it may be exploiting shortcuts rather than genuinely understanding the task. 
More broadly, the limited number of LLMs recognizing the importance of $f_{\mathrm{tc}}$ underscores a common deficiency in current LLMs' grasp of token causality.

\begin{figure}[t]
  \centering
  \includegraphics[width=\columnwidth]{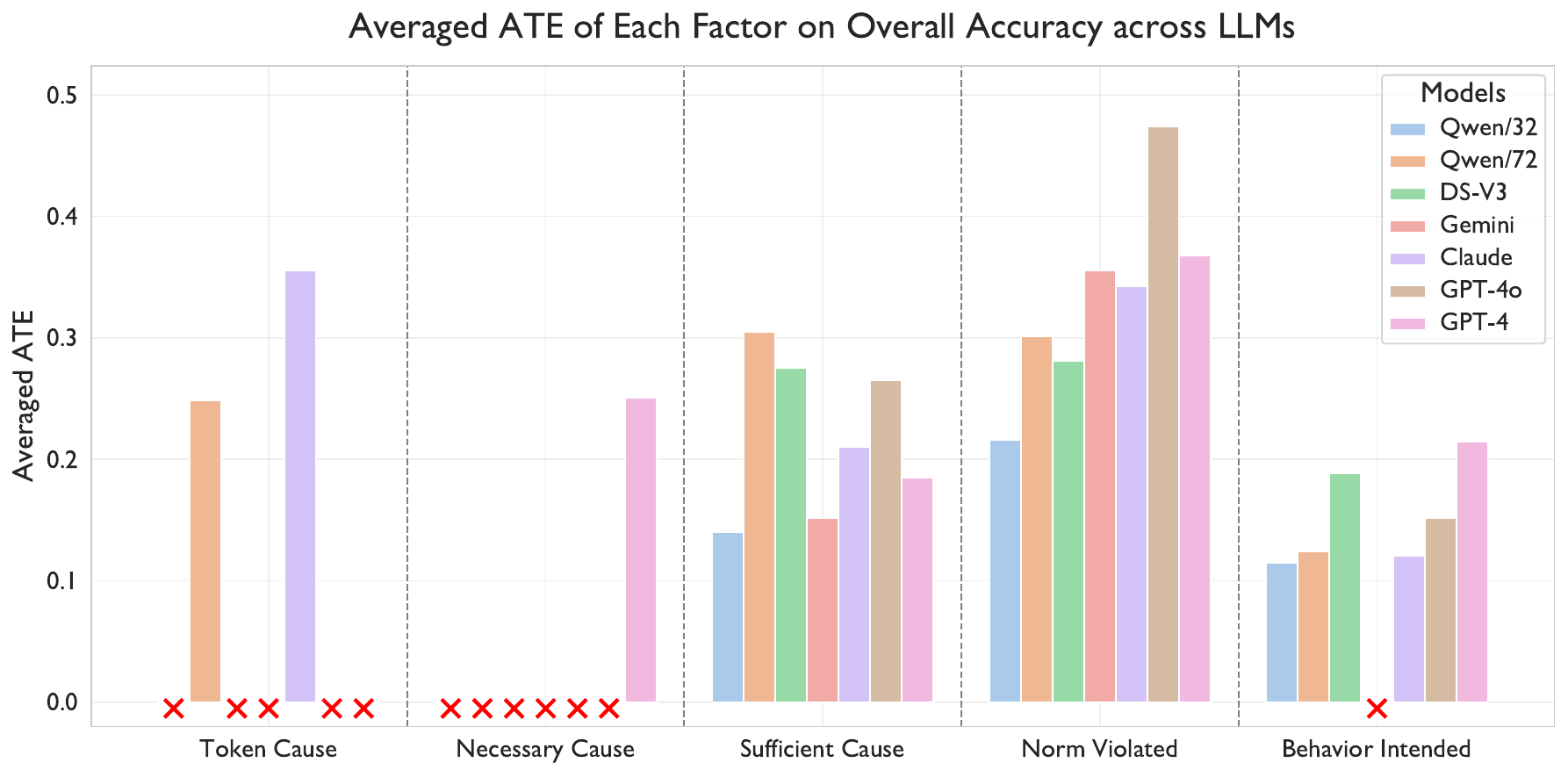}
  \caption{Results of the causal analysis.}
  \label{fig:causal_analysis_results}
\end{figure}

\subsection{Ablation Study}\label{sec:ablation_study}

We conduct an ablation study to quantify the contribution of each stage and report performance across the following settings: 1) \textit{vanilla}; 2) \textit{FO} (the first one stage only); 3) \textit{FT} (the first two stages only); and 4) \textit{\ourmodel}.

\paragraph{Non-Reasoning LLMs.}
1) \textit{The first stage alone typically degrades performance.}
This is observed across non-reasoning LLMs, with the exception of DeepSeek-V3.
While these models can effectively identify events, as shown in Tab. \ref{tab:main_results}, the instruction of the first stage may lead them to extract superfluous or irrelevant causal events, thereby hindering accurate decision-making.
2) \textit{Combining the first two stages typically improves performance.}
Encouraging non-reasoning LLMs to infer and consider distinct causal and cognitive factors promotes a clearer representation of the underlying causal mechanisms, leading to more informed and consistent reasoning.
3) \textit{Theory-guided reasoning yields the most substantial gains.}
This observation directly supports our central hypothesis: the proposed Alg. \ref{alg:theory_guided_reasoning} enables precise human-like causal reasoning by leveraging the structured causal knowledge inferred in the earlier stages, thereby aligning model outputs more closely with human judgments.

\paragraph{Reasoning LLMs.}
Since non-reasoning LLMs lack intrinsic reflection and verification mechanisms, we further evaluate slow thinking models including DeepSeek-R1 and QwQ-32B.
1) \textit{Slow thinking models perform poorly alone.} This underscores the inherent difficulty of human-like causal reasoning in LLMs, indicating that mere self-reflection without guidance is insufficient.
2) \textit{Slow thinking models achieve the largest performance gains with \ourmodel.} This demonstrates that reflection and verification capabilities become truly effective only when properly guided by theory.

\begin{table}[t]
\centering
\scriptsize
\caption{Ablation results on \ourbench. Green and red highlights denote performance rise/drop compared with Vanilla.}
\resizebox{\columnwidth}{!}{
\begin{tabular}{lc}
\toprule
\multicolumn{1}{l}{\textbf{Model}} & \multicolumn{1}{c}{\textbf{Vanilla $\rightarrow$ FO $\rightarrow$ FT $\rightarrow$ \ourmodel}} \\
\midrule
\multicolumn{2}{l}{\textit{Open-source LLMs}} \\
\midrule
\multicolumn{1}{l}{Qwen2.5-32B} & \multicolumn{1}{c}{62.67\% $\rightarrow$ \textcolor{red}{61.67\%} $\rightarrow$ \textcolor{green}{64.59\%} $\rightarrow$ \textcolor{green}{64.78\%}} \\
\multicolumn{1}{l}{Qwen2.5-72B} & \multicolumn{1}{c}{64.87\% $\rightarrow$ \textcolor{red}{64.41\%} $\rightarrow$ \textcolor{green}{67.98\%} $\rightarrow$ \textcolor{green}{67.52\%}} \\
\multicolumn{1}{l}{DeepSeek-V3} & \multicolumn{1}{c}{63.49\% $\rightarrow$ \textcolor{green}{67.25\%} $\rightarrow$ \textcolor{green}{65.87\%} $\rightarrow$ \textcolor{green}{67.61\%}} \\
\midrule
\multicolumn{2}{l}{\textit{Closed-source LLMs}} \\
\midrule
\multicolumn{1}{l}{Gemini-2.0} & \multicolumn{1}{c}{60.20\% $\rightarrow$ \textcolor{red}{58.19\%} $\rightarrow$ \textcolor{green}{60.84\%} $\rightarrow$ \textcolor{green}{64.96\%}} \\
\multicolumn{1}{l}{Claude-3.5} & \multicolumn{1}{c}{63.68\% $\rightarrow$ \textcolor{red}{58.83\%} $\rightarrow$ \textcolor{red}{60.93\%} $\rightarrow$ \textcolor{green}{70.54\%}} \\
\multicolumn{1}{l}{GPT-4o} & \multicolumn{1}{c}{58.65\% $\rightarrow$ \textcolor{red}{58.28\%} $\rightarrow$ \textcolor{red}{58.46\%} $\rightarrow$ \textcolor{green}{68.07\%}} \\
\multicolumn{1}{l}{GPT-4} & \multicolumn{1}{c}{63.77\% $\rightarrow$ \textcolor{red}{62.31\%} $\rightarrow$ \textcolor{green}{64.32\%} $\rightarrow$ \textcolor{green}{71.82\%}} \\
\midrule
\multicolumn{2}{l}{\textit{Reasoning LLMs}} \\
\midrule
\multicolumn{1}{l}{QwQ-32B} & \multicolumn{1}{c}{54.99\% $\rightarrow$ \textcolor{green}{56.17\%} $\rightarrow$ \textcolor{green}{56.08\%} $\rightarrow$ \textcolor{green}{67.15\%}} \\
\multicolumn{1}{l}{DeepSeek-R1} & \multicolumn{1}{c}{57.09\% $\rightarrow$ \textcolor{green}{58.28\%} $\rightarrow$ \textcolor{green}{58.92\%} $\rightarrow$ \textcolor{green}{69.72\%}} \\
\bottomrule
\end{tabular}
}
\label{tab:ablation_results}
\end{table}
\section{Conclusion}\label{sec:conclusion}

In this work, we focus on human-like causal reasoning in LLMs. We construct \ourbench, the first benchmark focusing on fine-grained human-like causal reasoning evaluation. It contains 1,093 challenging instances with detailed annotations of human reasoning steps.
Also, we propose \ourmodel, which systematically integrates LLMs with token causality and cognitive science theories for human-like causal reasoning.
Results reveal that \ourmodel consistently improves LLM performance over baselines, and that integrating theories into LLMs is highly effective.
In the future, it is interesting to explore how to internalize this ability into LLMs.

\section*{Limitations}\label{app:limitations}
\textit{First, the \ourmodel framework remains incomplete.} 
It incorporates several commonly studied cognitive factors but may overlook more fine-grained or domain-specific factors that influence human causal reasoning.
Furthermore, the current implementation processes only a single causal event at a time.
\textit{Second, relying solely on LLMs to identify token causes is insufficient.}
A more formal symbolic reasoning component is needed to ensure interpretability, consistency, and adherence to established causal principles.

\section*{Acknowledgements}
This work was supported by the National Natural Science Foundation of China (No. 62576016). We would like to thank the anonymous reviewers for their helpful comments.


\bibliography{custom}

\clearpage
\appendix

\clearpage

\section{Background}\label{sec:background}

\subsection{Human Causal Reasoning}\label{sec:human_causal_reasoning}
The causal reasoning process of humans is a sophisticated cognitive process. It moves beyond simple counterfactual analysis to incorporate the psychological nuances of human judgment \cite{knobe2008causal, alicke2000culpable}.

\paragraph{Stage 1.}
In this stage, \textit{humans establish a basic causal chain}.
For example, in philosophy and legal theory, individuals identify the but-for causes (Def. \ref{def:necessary_cause}) of an outcome \cite{hart1985causation}. 
This process creates a foundational set of potential causes without evaluating their relative importance.

\paragraph{Stage 2.}
In this stage, \textit{psychological and normative factors modulate the initial causal structure to produce a final judgment}. This is where human reasoning diverges from purely logical models. The strength of a perceived causal link is adjusted based on: 
1) \textit{normative considerations}, where people attribute greater causal responsibility to agents and events that violate social or moral norms \cite{knobe2008causal,hitchcock2009cause};
2) \textit{blame-driven evaluations}, where the desire to assign blame influences causal assessment for negative events. This motivational bias, central to attribution theory \cite{weiner1985attributional}, can retroactively alter perceptions of an agent's control and role, making their actions appear more blameworthy than an objective analysis would suggest \cite{alicke2000culpable, lagnado2008judgments}.

\subsection{Other Causal Factors}\label{sec:other_causal_factors}
Apart from the HP definition described in Section \ref{sec:token_causality}, we also include necessary cause (Def. \ref{def:necessary_cause}) and sufficient cause (Def. \ref{def:sufficient_cause}) as causal factors.

\begin{definition}[But-for/Necessary Cause]\label{def:necessary_cause}
    $\bm{X}=\bm{x}$ is a but-for cause of $\varphi$ if $\varphi$ would not have happened but for $\bm{X}=\bm{x}$.
\end{definition}

\begin{definition}[Sufficient Cause]\label{def:sufficient_cause}
    $\bm{X}=\bm{x}$ is a sufficient cause of $\varphi$ in the causal setting
    $(M,\bm{u})$ if the following four conditions hold:
    \begin{enumerate}[label=SC\arabic*., leftmargin=*]
        \item $(M,\bm{u})\models(\bm{X}=\bm{x})$ and
        $(M,\bm{u})\models\varphi$.
        
        \item Some conjunct of $\bm{X}=\bm{x}$ is part of a cause of
        $\varphi$ in the causal setting $(M,\bm{u})$.
        
        \item $(M,\bm{u}')\models[\bm{X}\leftarrow\bm{x}]\varphi$
        for all contexts $\bm{u}'$.
        
        \item $\bm{X}$ is minimal; there is no strict subset $\bm{X}'$
        of $\bm{X}$ such that $\bm{X}'=\bm{x}'$ satisfies conditions SC1,
        SC2, and SC3, where $\bm{x}'$ is the restriction of $\bm{x}$ to
        the variables in $\bm{X}'$.
    \end{enumerate}
\end{definition}

Prop. \ref{prop:necessary_token} formally proves that a token cause is an extension of a but-for cause.

\begin{proposition}\label{prop:necessary_token}
If $\bm{X}=\bm{x}$ is a but-for cause of $\varphi$ in the causal setting $(M,\bm{u})$, then $\bm{X}=\bm{x}$ is a cause of $\varphi$ according to the HP definition.
\end{proposition}

\begin{proof}
Suppose $\bm{X}=\bm{x}$ is a but-for cause of $\varphi$. Then there exists at least a possible value $\bm{x}'$ such that $(M,\bm{u}) \models [\bm{X} \leftarrow \bm{x}'] \lnot \varphi$.
Let $\bm{W} = \emptyset$ and consider the intervention $\bm{X} \leftarrow \bm{x}'$. Then AC2 is satisfied.
Therefore, $\bm{X} = \bm{x}$ qualifies as an token cause, as a special case where $\bm{W} = \emptyset$.
\end{proof}

\subsection{A Failure Case of But-for Causation}\label{sec:a_failure_case_of_but_for_causation}
A classic example of preemption \cite{lewis2000causation} is as follows:
\begin{quote}
    \textit{Suzy and Billy both pick up rocks and throw them at a bottle. Suzy's rock gets there first, shattering the bottle. Because both throws are perfectly accurate, Billy's would have shattered the bottle had it not been preempted by Suzy's throw.}
\end{quote}

According to Def. \ref{def:necessary_cause}, neither Suzy's nor Billy's throw would be identified as the cause of the bottle shattering.
However, token causality (Def. \ref{def:token_cause}) can capture this nuance.

\begin{proof}
We construct a structural causal model with the following variables:
\begin{itemize}
    \item $ST$ for ``Suzy throws,'' with values 0 (she does not) and 1 (she does);
    \item $BT$ for ``Billy throws,'' with values 0 (he does not) and 1 he does);
    \item $SH$ for ``Suzy's rock hits the bottle,'' with values 0 (it does not) and 1 (it does);
    \item $BH$ for ``Billy's rock hits the bottle,'' with values 0 (it does not) and 1 (it does);
    \item $BS$ for ``Bottle shatters,'' with values 0 (it does not) and 1 (it does).
\end{itemize}
The corresponding structural equations are defined as:
\begin{itemize}
    \item $BS=1$ if $SH=1$ or $BH=1$;
    \item $SH=1$ if $ST=1$;
    \item $BH=1$ if $BT=1$ and $SH=0$.
\end{itemize}
\begin{figure}[htbp]
    \centering
    \begin{tikzpicture}[>=Stealth, node distance=2cm]
        \node (ST) at (0,3) {$ST$};
        \node (BT) at (3,3) {$BT$};
        \node (SH) at (0,1.5) {$SH$};
        \node (BH) at (3,1.5) {$BH$};
        \node (BS) at (1.5,0) {$BS$};
        \draw[->, thick] (ST) -- (SH);
        \draw[->, thick] (BT) -- (BH);
        \draw[->, thick] (SH) -- (BH);
        \draw[->, thick] (SH) -- (BS);
        \draw[->, thick] (BH) -- (BS);
    \end{tikzpicture}
    \caption{The causal representation of the case.}
    \label{fig:preemption}
\end{figure}
First, it is clear that both $ST=1$ and $BS=1$ satisfy AC1 and AC3.
For $ST=1$, if Suzy had not thrown ($ST\rightarrow 0$) and $\bm{W}=\{BH\}$, then the bottle would not have shattered ($BS=0$) since we have forced Billy's rock not to hit the bottle ($BH\rightarrow 0$). Thus, $ST=1$ is an token cause of $BS=1$.
For $BT=1$, if Billy had not thrown ($BT\rightarrow 0$), all other variables would remain unchanged, including $BS$. Therefore, $BT=1$ is not an token cause of $BS=1$ since no suitable $\bm{W}$ can be found to satisfy AC2.
Therefore, according to Def. \ref{def:token_cause}, $ST=1$ is a cause of $BS=1$, whereas $BT=1$ is not.
\end{proof}

\section{Details of \ourmodel}\label{sec:details_of_framework}
\subsection{Details of Algorithm \ref{alg:theory_guided_reasoning}}\label{sec:details_of_algorithm}

\subsubsection{Reasoning Process Partitioning}\label{sec:reasoning_process_partitioning}

Let $E$ be a queried causal event and $O$ be the outcome event.
We use factors $f_{\mathrm{tc}}^E$, $f_{\mathrm{nc}}^E$, and $f_{\mathrm{sc}}^E$ to partition the reasoning process.
The objective of partitioning is to ensure that, within each partition, the causal relation between $E$ and $O$ can be assessed using existing techniques from token causality and causal judgment.

\paragraph{Criterion 1.}
If $E$ is a sufficient cause of $O$, the partitioning is based on $f_{\mathrm{nc}}^E$:
1) If $f_{\mathrm{nc}}^E$ holds, by Prop. \ref{prop:necessary_token}, $f_{\mathrm{tc}}^E$ also holds. Since $f_{\mathrm{sc}}^E$ also holds, $E$ has 100\% responsibility for $O$, making the causal relation determined. 
2) If $\lnot f_{\mathrm{nc}}^E$ holds, this corresponds to \textit{preemption} \citep{halpern2005causes,halpern2016actual} or \textit{overdetermination} \citep{schaffer2003overdetermining}, which are well-studied and defined primarily based on $f_{\mathrm{nc}}$ rather than $f_{\mathrm{tc}}$.

\paragraph{Criterion 2.}
If $E$ is not a sufficient cause of $O$, the partitioning is based on $f_{\mathrm{tc}}^E$:
1) If $f_{\mathrm{tc}}^E$ holds, $E$ is at least part of a cause by Def. \ref{def:token_cause}. Therefore, we only need to consider factors that increase causal strength.
2) If $\lnot f_{\mathrm{tc}}^E$ holds, by Prop. \ref{prop:necessary_token}, $\lnot f_{\mathrm{nc}}^E$ also holds. Since $\lnot f_{\mathrm{sc}}^E$ also holds, $E$ has 0\% responsibility for $O$, making the causal relation determined.

The partitions are exhaustive, as they jointly cover the entire space of possibilities. Furthermore, within each partition, the causal relation between $E$ and $O$ can be effectively assessed using existing theories from token causality and causal judgment. Therefore, the partitioning process is well-founded.

\subsubsection{Rules within Each Partition}\label{sec:rules_within_each_partition}

Then, we will justify how to assess the causal relation within each partition using rules.

\paragraph{Partition 1 (Line 3).}
Since $E$ is a necessary cause of $O$, it is also a token cause by Prop. \ref{prop:necessary_token}. Therefore, $E$ is a sufficient token cause (100\% responsibility) of $O$.
\begin{flalign*}
&f_{\mathrm{nc}}^E \land f_{\mathrm{sc}}^E \Rightarrow f_{\mathrm{tc}}^E \land f_{\mathrm{sc}}^E \\
&\Rightarrow E \text{ has 100\% responsibility for } O \\
&\Rightarrow E \text{ is a cause of } O.&
\end{flalign*}

\paragraph{Partition 2 (Line 17).}
Since $E$ is not a token cause of $O$, it is also not a necessary cause by Prop. \ref{prop:necessary_token}. Therefore, $E$ is neither a sufficient cause nor a necessary cause (0\% responsibility) of $O$.
\begin{flalign*}
&\lnot f_{\mathrm{tc}}^E \land \lnot f_{\mathrm{sc}}^E \Rightarrow \lnot f_{\mathrm{nc}}^E \land \lnot f_{\mathrm{sc}}^E \\
&\Rightarrow E \text{ has 0\% responsibility for } O \\
&\Rightarrow E \text{ is not a cause of } O.&
\end{flalign*}

\paragraph{Partition 3 (Lines 4-10).}
Since $E$ is a sufficient cause but not a necessary cause of $O$, there should exist a set of alternative sufficient causal events $\mathcal{C}_s$, each of which can independently cause $O$. 
In token causality, typically only one cause actually takes effect \citep{halpern2016actual}.
If events in $\mathcal{C}_s$ do not occur simultaneously, the one that occurs first is judged as the preemptive cause of $O$, while the others are not \citep{halpern2005causes,halpern2016actual}.
Otherwise, if events in $\mathcal{C}_s$ occur simultaneously, their responsibilities should be compared. Here, responsibility is defined as the relative degree to which $E$ causally contributes to $O$ compared with other contributing causal events. Causal contribution is evaluated using the remaining factors, i.e., $f_{\mathrm{resp}}^E\leftarrow \Phi (f_{\mathrm{n}}^E,f_{\mathrm{i}}^E,f_{\mathrm{order}}^E)$.
If $E$ has the highest responsibility for $O$, it is a cause of $O$ since it actually takes effect; otherwise, it is not. In the special case where all events in $\mathcal{C}_s$ have equal responsibility, the scenario corresponds to overdetermination \citep{schaffer2003overdetermining}, where each event in $\mathcal{C}_s$ is a cause of $O$.
Let $\mathcal{C}_s$ be the set of causal events that satisfy $\lnot f_{\mathrm{nc}} \land f_{\mathrm{sc}}$ (including $E$), then:
\begin{flalign*}
&\exists E_s \in \mathcal{C}_s,f_{\mathrm{order}}^{E_s} < f_{\mathrm{order}}^E \\
&\Rightarrow E \text{ does not actually take effect} \\
&\Rightarrow E \text{ is not a cause of } O.\\[0.6em]
&\forall E_s \in \mathcal{C}_s,f_{\mathrm{order}}^{E_s} \geq f_{\mathrm{order}}^E \\
&\Rightarrow E \text{ actually takes effect} \\
&\Rightarrow E \text{ is a cause of } O.\\[0.6em]
&\forall E_s \in \mathcal{C}_s,f_{\mathrm{order}}^{E_s} = f_{\mathrm{order}}^E\land f_{\mathrm{resp}}^E=\max_{E_s\in\mathcal{C}_s}f_{\mathrm{resp}}^{E_s} \\
&\Rightarrow E \text{ actually takes effect} \\
&\Rightarrow E \text{ is a cause of } O.\\[0.6em]
&\forall E_s \in \mathcal{C}_s,f_{\mathrm{order}}^{E_s} = f_{\mathrm{order}}^E\land f_{\mathrm{resp}}^E\neq\max_{E_s\in\mathcal{C}_s}f_{\mathrm{resp}}^{E_s} \\
&\Rightarrow E \text{ does not actually take effect} \\
&\Rightarrow E \text{ is not a cause of } O.&
\end{flalign*}

\paragraph{Partition 4 (Lines 11-16).}
Since $E$ is a token cause of $O$, it is at least a partial cause by Def. \ref{def:token_cause}. To determine whether $E$ is a cause of $O$, we first consider factors that enhance causal strength, including norm violations and intended behaviors. If $E$ satisfies either of these conditions, it is a cause of $O$. If neither condition is met, we turn to its responsibility $f_{\mathrm{resp}}^E\leftarrow\Phi(f_{\mathrm{order}}^E)$. If $E$ is uniquely the most responsible for $O$, it is a cause of $O$; otherwise, it is not.
Let $\mathcal{C}_a$ be the set of causal events that satisfy $f_{\mathrm{tc}}^E \land \lnot f_{\mathrm{sc}}$ (including $E$), then:
\begin{flalign*}
&f_{\mathrm{tc}}^E \land (f_{\mathrm{n}}^E \lor f_{\mathrm{i}}^E) \\
&\Rightarrow \text{The causal strength of } E \text{ is enhanced} \\
&\Rightarrow E \text{ is a cause of } O.& \\[0.6em]
&f_{\mathrm{tc}}^E \land \lnot (f_{\mathrm{n}}^E \lor f_{\mathrm{i}}^E)\land \forall E_a\in\mathcal{C}_a\setminus\{E\},f_{\mathrm{resp}}^E>f_{\mathrm{resp}}^{E_a} \\
&\Rightarrow \text{The causal strength of } E \text{ is enhanced} \\
&\Rightarrow E \text{ is a cause of } O.& \\[0.6em]
&f_{\mathrm{tc}}^E \land \lnot (f_{\mathrm{n}}^E \lor f_{\mathrm{i}}^E)\land \exists E_a\in\mathcal{C}_a,f_{\mathrm{resp}}^E\leq f_{\mathrm{resp}}^{E_a} \\
&\Rightarrow \text{The causal strength of } E \text{ is not enhanced} \\
&\Rightarrow E \text{ is not a cause of } O.&
\end{flalign*}

\subsection{Details of Explanation Generation}\label{sec:details_of_explanation_generation}
In Alg. \ref{alg:theory_guided_reasoning}, an explanation is generated for each answer. Most explanations are produced using predefined templates. For responsibility-related decisions, the explanations are generated by the LLM responsible for determining responsibility.
First, we present the templates:

\begin{itemize}
    \item Line 3: $E$ is a cause of $O$, since $E$ is a sufficient token cause.
    \item Line 7: $E$ is a cause of $O$, since $E$ occurs the earliest in a disjunctive causal structure. ($E$ is not a cause of $O$, since $E$ does not occur the earliest in a disjunctive causal structure.)
    \item Line 12: $E$ is a cause of $O$, since $E$ is a token cause and it violates a norm or is an intended behavior of an agent. ($E$ is not a cause of $O$, since $E$ is a token cause but it does not violate a norm and is not an intended behavior of an agent.)
    \item Line 17: $E$ is not a cause of $O$, since $E$ is not in the causal chain.
\end{itemize}

Second, we provide examples of LLM-generated responsibility-related explanations:

\begin{itemize}
    \item Yes, the focal causal event ``Billy arrives at 9 am'' is equally responsible relative to the other causal event ``Suzy arrives at 9 am''. Both Billy and Suzy arrived at the same time, triggering the motion detector. Although Billy was explicitly told not to come, Suzy was also explicitly told to come at exactly 9 am, making both of them equally responsible for the outcome.
    \item No. In the story, Drew, Kylie, and Oliver all ordered coffee on the same day, and any one of them ordering coffee would have been sufficient for the coffee shop to make a profit. However, Drew's action of ordering coffee on a Tuesday was unexpected and violated the norm, as he usually doesn't order coffee on Tuesdays. This makes Drew's action more noticeable, but it doesn't make him more responsible for the coffee shop's profit than Kylie or Oliver, who also ordered coffee.
    \item Yes, the focal causal event ``Brother steps on a slippery patch'' is more responsible relative to other causal events in the list. This is because it occurs later in the sequence of events, closer to the outcome event. The brother stepping on the slippery patch directly leads to him knocking over the lamp, which then fractures the water pipe, causing the damage. The other events, while they set the stage for this to happen, are further removed from the outcome event in terms of temporal order.
    \item No. Based on the story, Alice's action of logging in at 9.30 am is not more responsible than Zoe's action of logging in at 9 am for the outcome event. This is because both actions are necessary for the outcome to occur, and Zoe's action, which violates the company policy, happens first. Therefore, Zoe's action has a higher temporal order and is considered more responsible for the outcome.
\end{itemize}

\section{Details of \ourbench}\label{sec:details_of_benchmark}

\subsection{Details of Data Cleaning}\label{sec:details_of_data_cleaning}
We provide examples of erroneous, irrelevant, and partially flawed samples.

\paragraph{Erroneous Sample.}
Erroneous samples are those that miss critical background information. In the case below, the ``motion detector'' does not appeared in the story background.

\begin{quote}
\textbf{Story:} \textit{Suzy and Billy are working on a project that is very important for our nation's security. The boss tells Suzy: ``Be sure that you are here at exactly 9 am. It is absolutely essential that you arrive at that time.'' Then he tells Billy: ``Be sure that you do not come in at all tomorrow morning. It is absolutely essential that you not appear at that time.'' Both Billy and Suzy arrive at 9 am.}

\textbf{Query:} Did Billy cause the \textcolor{red}{motion detector} to go off?

\textbf{Answer:} No
\end{quote}

\paragraph{Irrelevant Sample.}
Irrelevant samples are those that query about intention.

\begin{quote}
\textbf{Story:} \textit{The CEO of a company is sitting in his office when his Vice President of R\&D comes in and says, ``We are thinking of starting a new programme. It will help us increase profits, but it will also harm the environment.'' The CEO responds that he doesn't care about harming the environment and just wants to make as much profit as possible. The programme is carried out, profits are made and the environment is harmed.}

\textbf{Query:} Did the CEO \textcolor{red}{intentionally} harm the environment?

\textbf{Answer:} Yes
\end{quote}

\paragraph{Partially Flawed Sample.}
Partially flawed samples are those with issues in the story, query, or answer. In many cases, errors in the query and answer co-occur. For instance, a duplicated or malformed query often results in an incorrect answer. In the case below, ``work emails'' is mistakenly written as ``spam emails.''

\begin{quote}
\textbf{Story:} \textit{Billy and Suzy work for the same company. They work in different rooms, and both of them sometimes need to access the central computer of the company. Nobody at the company is aware that if two people are logged into the central computer at the same time, some \textcolor{red}{spam} \textcolor{red}{emails} containing important customer information are immediately deleted from the central computer. In order to make sure that two people are available to answer phone calls during designated calling hours, the company issued the following official policy: Billy and Suzy are both permitted to log into the central computer in the mornings, and neither of them are permitted to log into the central computer in the afternoons. Today at 9 am, Billy and Suzy both log into the central computer at the same time. Immediately, some work emails containing important customer information are deleted from the central computer.}

\textbf{Query:} Did Suzy cause the central computer to delete some work emails containing important customer information?

\textbf{Answer:} Yes
\end{quote}

\subsection{Details of Data Annotation}\label{sec:details_of_data_annotation}
\begin{table}[h!]
\caption{Annotation guideline.}
\small
\label{tab:annotation_guideline}
\resizebox{\columnwidth}{!}{
\begin{tabular}{l l p{0.7\columnwidth}}
\toprule
\textbf{Factor} & \textbf{Value} & \textbf{Definition} \\
\midrule

$f_{\mathrm{occur}}^E$
 & \textsc{True}  & $E$ actually occurs in the causal setting. \\
 & \textsc{False} & Otherwise. \\
\midrule

$f_{\mathrm{order}}^E$
 & / & An integer starting from 0, denoting the temporal order of $E$ relative to other events in the causal setting. If two events occur simultaneously, they should share the same order. \\
\midrule

$f_{\mathrm{tc}}^E$
 & \textsc{True}  & $E$ is an token cause of $O$ by Def. \ref{def:token_cause}. If 1) $f_{\mathrm{occur}}^E\land f_{\mathrm{occur}}^O$ holds and 2) $f_{\mathrm{nc}}^E$ holds while allowing contingencies based on AC2, then $f_{\mathrm{tc}}^E$ holds. 
 \textcolor{red}{[Proof: AC1 holds if $f_{\mathrm{occur}}^E\land f_{\mathrm{occur}}^O$ holds; AC2 holds if $f_{\mathrm{nc}}^E$ holds while allowing contingencies; AC3 already holds since we only annotate minimal/atomic and non-overlapping causal events.]} \\
 & \textsc{False} & Otherwise. \\
\midrule

$f_{\mathrm{nc}}^E$
 & \textsc{True}  & $E$ is a necessary cause of $O$ by Def. \ref{def:necessary_cause}. If but for $E$, $O$ would not have occurred, then $f_{\mathrm{nc}}^E$ holds. \\
 & \textsc{False} & Otherwise. \\
\midrule

$f_{\mathrm{sc}}^E$
 & \textsc{True}  & $E$ is a sufficient cause of $O$ by Def. \ref{def:sufficient_cause}. If 1) $f_{\mathrm{tc}}^E$ holds and 2) $E$ always suffices to cause $O$ while changing the $f_{\mathrm{occur}}$ factor of other causal events, then $f_{\mathrm{sc}}^E$ holds. 
 \textcolor{red}{[Proof: SC1 holds if $f_{\mathrm{occur}}^E\land f_{\mathrm{occur}}^O$ holds; SC2 holds if $f_{\mathrm{tc}}^E$ holds; SC3 holds by counterfactual reasoning on causal events other than $E$; SC4 already holds since we only annotate minimal/atomic and non-overlapping causal events. Also, if $f_{\mathrm{tc}}^E$ holds, $f_{\mathrm{occur}}^E\land f_{\mathrm{occur}}^O$ holds.]} \\
 & \textsc{False} & Otherwise. \\
\midrule

$f_{\mathrm{n}}^E$
 & \textsc{True}  & $E$ violates a prescriptive or descriptive norm. \\
 & \textsc{False} & Otherwise. \\
\midrule

$f_{\mathrm{i}}^E$
 & \textsc{True}  & $E$ is an agent's behavior and the agent is aware of the potential consequences of their action and knowingly performs an action that leads to a foreseeable bad outcome. \\
 & \textsc{False} & Otherwise. \\ 
\bottomrule
\end{tabular}
}
\end{table}

The annotation criteria are presented in Tab. \ref{tab:annotation_guideline}. The annotation process follows the pipeline below:
\begin{enumerate}
    \item Annotate $\mathcal{C}$, $O$, $f_{\mathrm{occur}}$, $f_{\mathrm{order}}$, $f_{\mathrm{n}}$, and $f_{\mathrm{i}}$. The values of these factors are usually explicit in the story and query, so annotation is relatively straightforward.
    \item Annotate $f_{\mathrm{tc}}$, $f_{\mathrm{nc}}$, and $f_{\mathrm{sc}}$. Annotation follows the order $f_{\mathrm{nc}}\rightarrow f_{\mathrm{tc}}\rightarrow f_{\mathrm{sc}}$, because $f_{\mathrm{nc}}$ is the easiest to determine, and if it holds, $f_{\mathrm{tc}}$ necessarily holds; meanwhile, the determination of $f_{\mathrm{sc}}$ depends on $f_{\mathrm{tc}}$. Only when $\lnot f_{\mathrm{nc}}$ holds does one need to evaluate $f_{\mathrm{tc}}$ according to the AC2 clause in Def. \ref{def:token_cause}.
\end{enumerate}

\section{More Results}\label{sec:more_results}

\subsection{Pilot Results}\label{sec:pilot_results}

\begin{table}[h!]
\centering
\caption{Complete results of the pilot study.}
\label{tab:pilot-full}
\resizebox{\columnwidth}{!}{
\begin{tabular}{@{}lccc@{}}
\toprule
\multicolumn{1}{l}{\textbf{Methods}} & \multicolumn{1}{c}{\textbf{Acc. (C.)}} & \multicolumn{1}{c}{\textbf{Acc. (I.)}} & \multicolumn{1}{c}{\textbf{Acc.}} \\ \midrule

\multicolumn{1}{l}{Human Average} & \multicolumn{1}{c}{-} & \multicolumn{1}{c}{-} & \multicolumn{1}{c}{69.60\textcolor{gray}{$\pm$0.00}} \\ \midrule

\multicolumn{1}{l}{Qwen2.5-32B-Instruct} & 65.25\textcolor{blue}{$\pm$0.32} & 80.43\textcolor{blue}{$\pm$0.00} & \multicolumn{1}{c}{68.98\textcolor{blue}{$\pm$0.24}} \\
\multicolumn{1}{l}{~~~~+ zero-shot CoT} & 61.70\textcolor{blue}{$\pm$0.78} & 75.22\textcolor{blue}{$\pm$1.74} & \multicolumn{1}{c}{65.03\textcolor{blue}{$\pm$0.49}} \\
\multicolumn{1}{l}{~~~~+ manual CoT} & 65.89\textcolor{blue}{$\pm$0.59} & 79.13\textcolor{blue}{$\pm$1.06} & \multicolumn{1}{c}{69.14\textcolor{blue}{$\pm$0.34}} \\
\multicolumn{1}{l}{~~~~+ \ourmodel} & 70.14\textcolor{blue}{$\pm$1.83} & 80.43\textcolor{blue}{$\pm$0.00} & \multicolumn{1}{c}{72.67\textcolor{blue}{$\pm$1.38}} \\ \midrule

\multicolumn{1}{l}{Qwen2.5-72B-Instruct} & 65.89\textcolor{blue}{$\pm$1.21} & 77.83\textcolor{blue}{$\pm$1.30} & \multicolumn{1}{c}{68.82\textcolor{blue}{$\pm$1.04}} \\
\multicolumn{1}{l}{~~~~+ zero-shot CoT} & 62.06\textcolor{blue}{$\pm$2.27} & 77.83\textcolor{blue}{$\pm$3.61} & \multicolumn{1}{c}{65.94\textcolor{blue}{$\pm$2.02}} \\
\multicolumn{1}{l}{~~~~+ manual CoT} & 63.12\textcolor{blue}{$\pm$2.01} & 79.13\textcolor{blue}{$\pm$1.06} & \multicolumn{1}{c}{67.06\textcolor{blue}{$\pm$1.59}} \\
\multicolumn{1}{l}{~~~~+ \ourmodel} & 71.63\textcolor{blue}{$\pm$0.90} & 79.13\textcolor{blue}{$\pm$1.44} & \multicolumn{1}{c}{73.48\textcolor{blue}{$\pm$0.49}} \\ \midrule

\multicolumn{1}{l}{DeepSeek-V3} & 66.45\textcolor{blue}{$\pm$1.96} & 77.61\textcolor{blue}{$\pm$1.70} & \multicolumn{1}{c}{69.20\textcolor{blue}{$\pm$1.76}} \\
\multicolumn{1}{l}{~~~~+ zero-shot CoT} & 67.38\textcolor{blue}{$\pm$2.73} & 74.57\textcolor{blue}{$\pm$3.37} & \multicolumn{1}{c}{69.14\textcolor{blue}{$\pm$1.82}} \\
\multicolumn{1}{l}{~~~~+ manual CoT} & 62.62\textcolor{blue}{$\pm$2.15} & 76.30\textcolor{blue}{$\pm$1.52} & \multicolumn{1}{c}{65.99\textcolor{blue}{$\pm$1.59}} \\
\multicolumn{1}{l}{~~~~+ \ourmodel} & 69.93\textcolor{blue}{$\pm$1.88} & 78.48\textcolor{blue}{$\pm$1.17} & \multicolumn{1}{c}{72.03\textcolor{blue}{$\pm$1.62}} \\ \midrule

\multicolumn{1}{l}{Claude-3.5-Sonnet} & 67.23\textcolor{blue}{$\pm$2.57} & 73.04\textcolor{blue}{$\pm$4.68} & \multicolumn{1}{c}{68.66\textcolor{blue}{$\pm$2.52}} \\
\multicolumn{1}{l}{~~~~+ zero-shot CoT} & 65.25\textcolor{blue}{$\pm$4.09} & 71.96\textcolor{blue}{$\pm$4.29} & \multicolumn{1}{c}{66.90\textcolor{blue}{$\pm$3.69}} \\
\multicolumn{1}{l}{~~~~+ manual CoT} & 59.65\textcolor{blue}{$\pm$3.21} & 84.57\textcolor{blue}{$\pm$2.27} & \multicolumn{1}{c}{65.78\textcolor{blue}{$\pm$2.37}} \\
\multicolumn{1}{l}{~~~~+ \ourmodel} & \uline{72.62}\textcolor{blue}{$\pm$1.11} & 75.43\textcolor{blue}{$\pm$4.24} & \multicolumn{1}{c}{73.32\textcolor{blue}{$\pm$1.54}} \\ \midrule

\multicolumn{1}{l}{GPT-4o-2024-11-20} & 54.89\textcolor{blue}{$\pm$1.53} & 85.22\textcolor{blue}{$\pm$1.90} & \multicolumn{1}{c}{62.35\textcolor{blue}{$\pm$1.29}} \\
\multicolumn{1}{l}{~~~~+ zero-shot CoT} & 61.99\textcolor{blue}{$\pm$1.20} & 77.39\textcolor{blue}{$\pm$3.25} & \multicolumn{1}{c}{65.78\textcolor{blue}{$\pm$1.15}} \\
\multicolumn{1}{l}{~~~~+ manual CoT} & 60.64\textcolor{blue}{$\pm$0.91} & 78.48\textcolor{blue}{$\pm$2.27} & \multicolumn{1}{c}{65.03\textcolor{blue}{$\pm$0.55}} \\
\multicolumn{1}{l}{~~~~+ \ourmodel} & 70.50\textcolor{blue}{$\pm$1.96} & 85.22\textcolor{blue}{$\pm$2.13} & \multicolumn{1}{c}{\uline{74.12}\textcolor{blue}{$\pm$1.36}} \\ \midrule

\multicolumn{1}{l}{GPT-4-0613} & 62.48\textcolor{blue}{$\pm$0.59} & 80.22\textcolor{blue}{$\pm$2.65} & \multicolumn{1}{c}{66.84\textcolor{blue}{$\pm$0.83}} \\
\multicolumn{1}{l}{~~~~+ zero-shot CoT} & 62.77\textcolor{blue}{$\pm$0.73} & 81.52\textcolor{blue}{$\pm$1.46} & \multicolumn{1}{c}{67.38\textcolor{blue}{$\pm$0.68}} \\
\multicolumn{1}{l}{~~~~+ manual CoT} & 65.89\textcolor{blue}{$\pm$0.92} & 78.26\textcolor{blue}{$\pm$0.97} & \multicolumn{1}{c}{68.93\textcolor{blue}{$\pm$0.65}} \\
\multicolumn{1}{l}{~~~~+ \ourmodel} & \textbf{74.61}\textcolor{blue}{$\pm$0.76} & 78.26\textcolor{blue}{$\pm$2.17} & \multicolumn{1}{c}{\textbf{75.51}\textcolor{blue}{$\pm$0.79}} \\ \bottomrule
\end{tabular}
}
\end{table}

\paragraph{Complete Results.}
The complete results are presented in Tab. \ref{tab:pilot-full}. 
We report both overall and fine-grained accuracies along with standard deviations over 10 runs.
\textit{First, for causation queries, all models exhibit consistent performance improvements with the integration of \ourmodel, with open-source LLMs generally benefiting more.} For instance, GPT-4's accuracy increase from 62.48\% to 74.61\% after applying \ourmodel. GPT-4 + \ourmodel also achieves the highest overall accuracy of 75.51\%, significantly surpassing the average human performance. These results suggest that incorporating domain theory into LLMs to enhance reasoning capabilities is promising.
\textit{Second, neither zero-shot nor manual CoT is sufficient for human-like causal reasoning.} This suggests the inherent difficulty of the task.

\paragraph{Significance Test.}
We have also conducted an approximate randomization test \cite{edgington1969approximate} to assess the statistical significance of \ourmodel on Big-Bench Hard causal judgment, since this dataset is relatively small in size. For each model, we evaluate \ourmodel and 3 baselines across 10 runs. Each test is repeated 30 times with 10,000 trials per run to compute the average $p$-value. A Bonferroni correction \cite{dunn1961multiple} is applied for multiple comparisons (18 in total: 6 models $\times$ 3 baselines), yielding a corrected significance threshold of approximately 0.00278. As shown in Tab. \ref{tab:statistical_significance}, \ourmodel achieves statistical significance in 18/18 comparisons and 16/18 comparisons after Bonferroni correction, despite the relatively small size of Big-Bench Hard causal judgment.

\begin{table*}[h!]
    \centering
    \caption{Statistical significance of \ourmodel on Big-Bench Hard causal judgment.}
    \label{tab:statistical_significance}
    \resizebox{\textwidth}{!}{%
    \begin{tabular}{l l c c c c}
        \toprule
        \textbf{Model} & \textbf{Comparison} & \textbf{Accuracy Diff} & \textbf{$p$-value ($\pm$ std)} & \textbf{Significant} & \textbf{Significant (Bonferroni)} \\
        \midrule
        \multirow{3}{*}{Qwen2.5-32B-Instruct} & \ourmodel vs. Vanilla & $0.0369$ & $0.0020 \pm 0.0004$ & $\surd$ & $\surd$ \\
        & \ourmodel vs. Zero-shot CoT & $0.0765$ & $0.0021 \pm 0.0004$ & $\surd$ & $\surd$ \\
        & \ourmodel vs. Manual CoT & $0.0353$ & $0.0022 \pm 0.0004$ & $\surd$ & $\surd$ \\
        \midrule
        \multirow{3}{*}{Qwen2.5-72B-Instruct} & \ourmodel vs. Vanilla & $0.0465$ & $0.0022 \pm 0.0004$ & $\surd$ & $\surd$ \\
        & \ourmodel vs. Zero-shot CoT & $0.0754$ & $0.0022 \pm 0.0003$ & $\surd$ & $\surd$ \\
        & \ourmodel vs. Manual CoT & $0.0642$ & $0.0020 \pm 0.0004$ & $\surd$ & $\surd$ \\
        \midrule
        \multirow{3}{*}{DeepSeek-V3} & \ourmodel vs. Vanilla & $0.0283$ & $0.0079 \pm 0.0009$ & $\surd$ & $\times$ \\
        & \ourmodel vs. Zero-shot CoT & $0.0289$ & $0.0020 \pm 0.0005$ & $\surd$ & $\surd$ \\
        & \ourmodel vs. Manual CoT & $0.0604$ & $0.0020 \pm 0.0003$ & $\surd$ & $\surd$ \\
        \midrule
        \multirow{3}{*}{Claude-3.5-Sonnet} & \ourmodel vs. Vanilla & $0.0465$ & $0.0080 \pm 0.0007$ & $\surd$ & $\times$ \\
        & \ourmodel vs. Zero-shot CoT & $0.0642$ & $0.0021 \pm 0.0005$ & $\surd$ & $\surd$ \\
        & \ourmodel vs. Manual CoT & $0.0754$ & $0.0020 \pm 0.0004$ & $\surd$ & $\surd$ \\
        \midrule
        \multirow{3}{*}{GPT-4o-2024-11-20} & \ourmodel vs. Vanilla & $0.1176$ & $0.0019 \pm 0.0005$ & $\surd$ & $\surd$ \\
        & \ourmodel vs. Zero-shot CoT & $0.0834$ & $0.0020 \pm 0.0005$ & $\surd$ & $\surd$ \\
        & \ourmodel vs. Manual CoT & $0.0909$ & $0.0021 \pm 0.0004$ & $\surd$ & $\surd$ \\
        \midrule
        \multirow{3}{*}{GPT-4-0613} & \ourmodel vs. Vanilla & $0.0866$ & $0.0021 \pm 0.0004$ & $\surd$ & $\surd$ \\
        & \ourmodel vs. Zero-shot CoT & $0.0813$ & $0.0020 \pm 0.0004$ & $\surd$ & $\surd$ \\
        & \ourmodel vs. Manual CoT & $0.0658$ & $0.0020 \pm 0.0005$ & $\surd$ & $\surd$ \\
        \bottomrule
    \end{tabular}
    }
\end{table*}

\subsection{Case Study}\label{sec:case_study}

\paragraph{Reasoning Steps.}
For the first axis of our case study, we present the reasoning steps of Claude to quantitatively assess the effectiveness of \ourmodel.
We first present an example sample.

\begin{quote}
\textbf{Story:} \textit{Tom has a huge garden and loves flowers. He employed two gardeners who take care of the plants on his 30 flower beds: Alex and Benni. Both can independently decide on their working hours and arrange who cares for which flower beds. Alex and Benni are very reliable and Tom is satisfied with their work. Nevertheless he wants to optimize the plant growth. Since Tom has read in a magazine that plants grow better when they are fertilized, he decides to let Alex and Benni fertilize his plants. The magazine recommends the use of the chemicals A X200R or B Y33R, since both are especially effective. However, Tom also read that it can damage plants when they are exposed to multiple different types of chemicals. Tom therefore decides that he only wants to use one fertilizer. He goes for A X200R. Tom instructs Alex and Benni to buy the chemical A X200R and to use only this fertilizer. Alex volunteers for buying several bottles of this chemical for Benni and himself. After a few weeks, Tom goes for a walk in his garden. He realizes that some of his plants are much prettier and bigger than before. However, he also realizes that some of his plants have lost their beautiful color and are dried up. That makes Tom very sad and reflective. He wonders whether the drying of his plants might have something to do with the fertilization. He wants to investigate this matter and talks to Alex and Benni. Alex tells him that he followed Tom's instruction: ``I only bought and used the chemical A X200R which I had funneled into the blue can.'' Benni suddenly is startled and says to Alex: ``What? You funneled A X200R into the blue can? But you told me you had funneled it into the green can! That's why I always used the green can!'' Alex replies: ``Did I? Then I am sorry!'' Tom remembers that he had filled B Y33R in a green can - long before he had read about the chemicals in his magazine. He had never used it. So Benni must have accidentally, without knowing it, applied the chemical B Y33R, whereas only Alex applied A X200R. Tom realizes that the plants dried up in the flower beds on which both A X200R and B Y33R were applied by the gardeners.}

\textbf{Query:} Did Benni cause the plant to dry out?

\textbf{Gold Answer:} No.
\end{quote}
In the example, there are three main causal events: 1) $E_1$, Alex's miscommunication about the can color; 2) $E_2$, Alex uses A X200R; and 3) $E_3$, Benni unknowingly uses B Y33R following Tom's instruction. The outcome event $O$ is ``The plants dry out.''
Here, $E_2=1\land E_3=1$ is a cause of $O=1$.
A potential source of confusion arises from $E_3$: one might mistakenly judge it to be an individual cause of $O$ because it appears to violate the ``norm'' of Tom's instruction. However, the norm in this context is the instruction provided by Alex that A X200R is in the green can and that Benni should use the fertilizer from the green can. Since Benni follows this instruction, $E_1$ does not involve a norm violation. Therefore, $E_1$ is merely a partial cause of $O$.
As shown in the reasoning steps, the vanilla, zero-shot CoT, and manual CoT settings exhibit similar patterns. On the one hand, Claude demonstrates the ability to identify events without explicit prompting and can also take into account factors such as intention. This indicates that the first two stages of \ourmodel are implicit capabilities of LLMs.
However, Claude falls short in two key areas under these settings: 1) \textit{It fails to comprehensively consider all relevant factors.} For instance, sufficiency and normality are entirely overlooked despite their importance in this example. 2) \textit{It lacks a structured mechanism to synthesize its causal analysis into a coherent decision.} Consequently, Claude disproportionately relies on naive counterfactual dependence (i.e., $f_{\mathrm{nc}}$), which undermines its reasoning quality.
In contrast, with \ourmodel, Claude performs all key reasoning steps correctly. It first identifies the queried causal event, ``Benni uses green can,'' then accurately infers the factor values for both conjunctive causal events: ``Benni uses green can'' and ``Alex uses blue can.'' Following this, it correctly traverses Alg. \ref{alg:theory_guided_reasoning} to reach the final decision. Specifically, ``Benni uses green can'' satisfies $\lnot f_{\mathrm{tc}} \land f_{\mathrm{sc}}$, does not satisfy $f_{\mathrm{n}} \lor f_{\mathrm{i}}$, and is not assigned the uniquely highest responsibility due to occurring later than other causal events. Thus, it leads to the correct answer.
Although Claude identifies three additional causal events, they are inconsequential, as all factors take on False. This indicates that these events can be safely removed from the causal setting without affecting the outcome.

\begin{figure*}[h!]
\begin{tcolorbox}[title=Reasoning Steps (Claude + Vanilla),width=\textwidth]
To answer this question, we need to carefully analyze the causal chain of events and consider the role of Benni's actions in the outcome.

Let's break down the key points:

\begin{itemize}
    \item[1.] Tom decided to use only one fertilizer (A X200R) to avoid damaging plants with multiple chemicals.
    \item[2.] Alex bought A X200R and funneled it into the blue can, but mistakenly told Benni it was in the green can.
    \item[3.] Benni, believing Alex's misinformation, used the fertilizer from the green can.
    \item[4.] The green can actually contained B Y33R, which Tom had filled long ago and never used.
    \item[5.] The plants dried up in flower beds where both A X200R and B Y33R were applied.
\end{itemize}

Benni's action of using the fertilizer from the green can (unknowingly B Y33R) was a crucial part of the causal chain that led to the plants drying out. If Benni had not applied B Y33R, the plants would not have been exposed to both chemicals, which was the condition that caused the drying.

However, it's important to note that Benni's action was based on misinformation from Alex and was unintentional. Benni believed he was following Tom's instructions and using the correct fertilizer.

Despite the unintentional nature of Benni's action, from a strict causal perspective, Benni's application of B Y33R was a necessary condition for the harmful outcome.

Therefore, the answer is:

Yes

Benni did cause the plants to dry out, albeit unintentionally and as part of a chain of events involving misinformation and misunderstanding.
\end{tcolorbox}
\end{figure*}

\begin{figure*}[h!]
\begin{tcolorbox}[title=Reasoning Steps (Claude + Zero-shot CoT),width=\textwidth]
To answer this question, let's break down the key events and analyze the causal chain:

1. Tom decided to use only one fertilizer (A X200R) to avoid damaging plants with multiple chemicals.
2. Alex bought A X200R and funneled it into a blue can.
3. Alex mistakenly told Benni that he had put A X200R in the green can.
4. Benni, trusting Alex's information, used the fertilizer from the green can.
5. The green can actually contained B Y33R, which Tom had filled long ago and never used.
6. The plants dried up in flower beds where both A X200R and B Y33R were applied.

Now, let's consider Benni's role in this causal chain:

\begin{itemize}
    \item[1.] Benni did not intentionally use B Y33R; he thought he was using A X200R based on Alex's misinformation.
    \item[2.] Benni's action of using the fertilizer from the green can was a necessary component in the causal chain that led to the plants drying out.
    \item[3.] If Benni had not used the fertilizer from the green can (B Y33R), the plants would not have been exposed to both chemicals and would not have dried out.
    \item[4.] While Benni's action was based on misinformation, he still played a crucial role in the outcome.
    \item[5.] In causal judgment, we often consider whether an action was a necessary part of the causal chain, regardless of the actor's intentions or knowledge.
\end{itemize}

Given these considerations, we can conclude that Benni did cause the plants to dry out, even though it was unintentional and based on misinformation.

Answer: Yes

Benni's actions were a necessary part of the causal chain that led to the plants drying out, even though he was not aware of the true nature of the fertilizer he was using. In causal analysis, the unintended consequences of an action can still be attributed to the actor if their action was necessary for the outcome to occur.
\end{tcolorbox}
\end{figure*}

\begin{figure*}[h!]
\begin{tcolorbox}[title=Reasoning Steps (Claude + Manual CoT),width=\textwidth]
To answer this question, let's analyze the key points from the story:

\begin{itemize}
    \item[1.] Tom instructed both gardeners to use only A X200R fertilizer.
    \item[2.] Alex bought A X200R and put it in the blue can.
    \item[3.] Benni mistakenly thought A X200R was in the green can, which actually contained B Y33R.
    \item[4.] The plants dried up in flower beds where both A X200R and B Y33R were applied.
    \item[5.] Benni unknowingly used B Y33R from the green can.
\end{itemize}

Given these facts, we can conclude that Benni did indeed cause the plants to dry out, albeit unintentionally. Here's why:

\begin{itemize}
    \item[1.] Benni's use of B Y33R, combined with Alex's use of A X200R, created the condition that led to the plants drying out.
    \item[2.] If Benni had not used the fertilizer from the green can (B Y33R), the plants would not have been exposed to both chemicals, which was the cause of the damage.
    \item[3.] Even though Benni's actions were accidental and based on misinformation, they still contributed directly to the outcome.
\end{itemize}

Therefore, the answer is: Yes

Benni caused the plants to dry out, even though it was unintentional and based on a misunderstanding. In causal reasoning, the intent is not always necessary for establishing causation; the action and its consequences are what matter most.
\end{tcolorbox}
\end{figure*}

\begin{figure*}[h!]
\begin{tcolorbox}[title=Reasoning Steps (Claude + \ourmodel),width=\textwidth]
\begin{verbatim}
{
    "causal_events": {
        "Tom instructs to use A X200R": {
            "occur": True, 
            "order": 0, 
            "focal": False,
            "sufficient": False, 
            "necessary": False, 
            "halpern_pearl": False, 
            "norm_violated": False, 
            "behavior_intended": False
        }, 
        "Alex buys A X200R": {
            "occur": True, 
            "order": 1, 
            "focal": False,
            "sufficient": False, 
            "necessary": False, 
            "halpern_pearl": False, 
            "norm_violated": False, 
            "behavior_intended": False
        }, 
        "Alex funnels A X200R into blue can": {
            "occur": True, 
            "order": 2, 
            "focal": False,
            "sufficient": False, 
            "necessary": False, 
            "halpern_pearl": False, 
            "norm_violated": False, 
            "behavior_intended": False
        }, 
        "Alex tells Benni about green can": {
            "occur": True, 
            "order": 3, 
            "focal": False,
            "sufficient": False,
            "necessary": True, 
            "halpern_pearl": True, 
            "norm_violated": False, 
            "behavior_intended": False
        }, 
\end{verbatim}
\end{tcolorbox}
\end{figure*}

\begin{figure*}[h!]
\begin{tcolorbox}[width=\textwidth]
\begin{verbatim}
        "Benni uses green can": {
            "occur": True, 
            "order": 4, 
            "focal": True,
            "sufficient": False, 
            "necessary": True, 
            "halpern_pearl": True, 
            "norm_violated": False, 
            "behavior_intended": False
        }, 
        "Alex uses blue can": {
            "occur": True, 
            "order": 4, 
            "focal": False,
            "sufficient": False, 
            "necessary": True, 
            "halpern_pearl": True, 
            "norm_violated": False, 
            "behavior_intended": False
        }
    }, 
    "outcome_event": {
        "Plants dry out": {
            "occur": True, 
            "order": 5
        }
    }
}
\end{verbatim}
Explanation:

No. Based on the story, ``Benni uses green can'' is not more responsible relative to the other events in the list. All three events (``Benni uses green can'', ``Alex tells Benni about green can'', and ``Alex uses blue can'') occur in sequence with ``Alex tells Benni about green can'' happening first (order 3), followed by both ``Benni uses green can'' and ``Alex uses blue can'' occurring simultaneously (both order 4). Since ``Alex tells Benni about green can'' occurs earlier, it would have more responsibility according to the temporal order criterion. Between ``Benni uses green can'' and ``Alex uses blue can'', they have equal responsibility since they occur simultaneously (same order 4).
\end{tcolorbox}
\end{figure*}

\paragraph{Seed vs. Generated Samples.}
As for the second axis, we present examples comparing seed samples with generated samples to demonstrate that \ourbench is both more challenging and diverse.
\textit{Through addition, new details with spurious correlations are introduced, potentially distracting models from the actual causal structure.} In Example 1, the story of the seed sample contains only events relevant to the causal setting. In contrast, the generated sample includes an additional, irrelevant event ``Sophia prepares the main course and appetizers.'' This extraneous detail may mislead models into inferring a spurious correlation between this event and the outcome, thereby increasing the reasoning difficulty.
\textit{Through removal, important but non-essential causal cues such as explicit conjunctive statements like ``When $E_1$ and $E_2$ occur, $O$ will occur.'' may be omitted.} In Example 2, the seed sample explicitly states that ``The machine will short circuit if both the black wire and the red wire touch the battery at the same time.'' In contrast, the generated sample omits this explicit conjunctive specification, instead conveying it only implicitly. Such omissions increase the reasoning complexity of the sample.
\textit{Through reorganization, the structure of the story may become more diverse.} In both Example 1 and 2, the seed sample and its corresponding generated sample differ in multiple aspects, including the story setting, the addition or removal of details, and the organization of individual sentences as well as the overall paragraph structure.
As a result, the generated samples in \ourbench pose greater challenges and exhibit higher diversity than their seed counterparts. This increased difficulty is also reflected in model performance: across all models, accuracy on causation queries declines in \ourbench compared with in Big-Bench Hard causal judgment, and the gains attributed to \ourmodel are reduced accordingly.

\begin{figure*}[!htbp]
\begin{tcolorbox}[title=Example 1: Addition leads to spurious correlations.]
\texttt{\# Seed Sample.}

\textbf{Story:} Louie is playing a game of basketball, and he made a bet with his friends who are watching on the sidelines. If Louie either makes a layup or makes a 3-point shot during the game, then he'll win \$100. Just when the game started, Louie immediately got the ball at the 3-point line. He looked to the basket, dribbled in, and then made a layup right at the beginning of the game. Louie and his friends continued playing, but as hard as he tried, Louie couldn't make another shot. And then right at the end of the game as the clock was winding down, Louie got the ball at the 3-point line. He looked to the basket, focused his shot, and made a 3-point shot right at the buzzer. Then the game ended. Because Louie would win \$100 if he either made a layup or a 3-point shot, Louie won \$100.

\textbf{Query:} Did Louie win the \$100 bet because he made the layup?

\textbf{Reasoning:}
\begin{verbatim}
{
    "causal_events": {
        "Louie makes a layup": {
            "occur": true, 
            "order": 0, 
            "focal": true,
            "sufficient": true, 
            "necessary": false, 
            "halpern_pearl": false, 
            "norm_violated": false, 
            "behavior_intended": true
        }, 
        "Louie makes a 3-point shot": {
            "occur": true, 
            "order": 1, 
            "focal": true,
            "sufficient": true, 
            "necessary": false, 
            "halpern_pearl": false, 
            "norm_violated": false, 
            "behavior_intended": true
        }
    }, 
    "outcome_event": {
        "Louie wins \$100": {
            "occur": true, 
            "order": 2
        }
    }
}
\end{verbatim}

\textbf{Answer:} Yes
\end{tcolorbox}
\end{figure*}

\begin{figure*}[!htbp]
\begin{tcolorbox}
\texttt{\# Generated Sample.}

\textbf{Story:} Sophia has invited friends over for a cozy dinner evening at her apartment, where she plans to showcase her cooking skills. The highlight of the night, according to her invitation, is the dessert, which she promised to prepare herself. The party will feature dessert if Sophia either bakes a cake or bakes cookies. Sophia decided to begin the day by focusing on the most intricate dessert: the cake. She spent hours meticulously measuring ingredients, frosting layers, and decorating the top with a swirl of fresh whipped cream. It turned out beautifully, and she proudly displayed it on the dining table. Feeling confident, Sophia dedicated her remaining time to preparing the main course and appetizers. As the guests began sending texts that they were on their way, Sophia realized she had just enough time to bake a quick batch of cookies. They were ready just as her first guests entered the apartment. With both options available, the party was a success, and dessert was served. Since the party would have dessert if Sophia either baked cookies or a cake, there was no shortage of delicious treats to enjoy.

\textbf{Query:} Did the dinner party have dessert because Sophia baked cookies?

\textbf{Reasoning:}
\begin{verbatim}
{
    "causal_events": {
        "Sophia bakes cookies": {
            "occur": true, 
            "order": 1, 
            "focal": true,
            "sufficient": true, 
            "necessary": false, 
            "halpern_pearl": false, 
            "norm_violated": false, 
            "behavior_intended": true
        }, 
        "Sophia bakes a cake": {
            "occur": true, 
            "order": 0,
            "focal": true,
            "sufficient": true, 
            "necessary": false, 
            "halpern_pearl": false, 
            "norm_violated": false, 
            "behavior_intended": true
        }
    }, 
    "outcome_event": {
        "The party has dessert": {
            "occur": true, 
            "order": 2
        }
    }
}
\end{verbatim}

\textbf{Answer:} No
\end{tcolorbox}
\end{figure*}

\begin{figure*}[!htbp]
\begin{tcolorbox}[title=Example 2: Removal leads to fewer explicit causal cues.]
\texttt{\# Seed Sample.}

\textbf{Story:} A machine is set up in such a way that it will short circuit if both the black wire and the red wire touch the battery at the same time. The machine will not short circuit if just one of these wires touches the battery. The black wire is designated as the one that is supposed to touch the battery, while the red wire is supposed to remain in some other part of the machine. One day, the black wire and the red wire both end up touching the battery at the same time. There is a short circuit.

\textbf{Query:} Did the black wire cause the short circuit?

\textbf{Reasoning:}
\begin{verbatim}
{
    "causal_events": {
        "The black wire touches the battery": {
            "occur": true,
            "order": 0,
            "focal": true,
            "sufficient": false,
            "necessary": true,
            "halpern_pearl": true,
            "norm_violated": false,
            "behavior_intended": false
        },
        "The red wire touches the battery": {
            "occur": true,
            "order": 0,
            "focal": false,
            "sufficient": false,
            "necessary": true,
            "halpern_pearl": true,
            "norm_violated": true,
            "behavior_intended": false
        }
    },
    "outcome_event": {
        "The machine short circuits": {
            "occur": true,
            "order": 1
        }
    }
}
\end{verbatim}

\textbf{Answer:} No
\end{tcolorbox}
\end{figure*}

\begin{figure*}[!htbp]
\begin{tcolorbox}
\texttt{\# Generated Sample.}

\textbf{Story:} In a water management system for a large greenhouse, there are two valves controlling water supply to a central storage tank. Valve A is part of the main operational circuit, intended to manage daily water flow. Valve B, on the other hand, serves as an emergency bypass valve that should remain shut to prevent mixing excess water into the system. Under normal circumstances, if just one valve is open, the tank functions properly without incident. One morning, during a routine inspection, both Valve A and Valve B are mistakenly left open for hours, allowing water from two separate sources to flow into the tank concurrently. As a result, the tank exceeds its capacity and floods the greenhouse floor, causing damage.

\textbf{Query:} Did Valve A cause the overflow?

\textbf{Reasoning:}
\begin{verbatim}
{
    "causal_events": {
        "Valve A is open": {
            "occur": true,
            "order": 0,
            "focal": true,
            "sufficient": false,
            "necessary": true,
            "halpern_pearl": true,
            "norm_violated": false,
            "behavior_intended": false
        },
        "Valve B is open": {
            "occur": true,
            "order": 0,
            "focal": false,
            "sufficient": false,
            "necessary": true,
            "halpern_pearl": true,
            "norm_violated": true,
            "behavior_intended": false
        }
    },
    "outcome_event": {
        "The storage tank overflows": {
            "occur": true,
            "order": 1
        }
    }
}
\end{verbatim}

\textbf{Answer:} No
\end{tcolorbox}
\end{figure*}

\begin{figure*}[h!]
\begin{tcolorbox}[title=Prompt for Baselines]
\textcolor{blue}{\textbf{Vanilla}} \\[0.3em]
\texttt{[SYSTEM] You are an expert in the field of token causality and causal judgment. Given the story and query of a logic-based causal judgment problem, you can effectively solve it.}

\textbf{Story}: \{story\}

\textbf{Query}: \{query\}

Answer (Yes or No?): \\[0.6em]
\textcolor{blue}{\textbf{Zero-shot CoT}} \\[0.3em]
\texttt{[SYSTEM] You are an expert in the field of token causality and causal judgment. Given the story and query of a logic-based causal judgment problem, you can effectively solve it.}

\textbf{Story}: \{story\}

\textbf{Query}: \{query\}

Let's think step by step.

Answer (Yes or No?): \\[0.6em]
\textcolor{blue}{\textbf{Manual CoT}} \\[0.3em]
\texttt{[SYSTEM] You are an expert in the field of token causality and causal judgment. Given the story and query of a logic-based causal judgment problem, you can effectively solve it.}

Here we will provide three chain-of-thought examplars, followed by a binary question that needs to be answered.

\textbf{Story}: \{story1\}

\textbf{Query}: \{query1\}

\textbf{Answer (with chain of thought)}: \{answer1\}

\textbf{Story}: \{story2\}

\textbf{Query}: \{query2\}

\textbf{Answer (with chain of thought)}: \{answer2\}

\textbf{Story}: \{story3\}

\textbf{Query}: \{query3\}

\textbf{Answer (with chain of thought)}: \{answer3\}

\textbf{Story}: \{story\}

\textbf{Query}: \{query\}

Answer (Yes or No?):
\end{tcolorbox}
\end{figure*}

\begin{figure*}
\begin{tcolorbox}[title=Prompt for \ourmodel (Causal Setting Establishment),width=\textwidth]
\texttt{[SYSTEM] You are an expert in the field of token causality and causal judgment. Given the story and query of a logic-based causal judgment problem, you can effectively assist in solving the problem following the instructions provided.}

\textbf{Story}: \{story\}

\textbf{Query}: \{query\}

Based on the story and query of a logic-based causal judgment problem, establish the causal setting as follows.

\begin{itemize}
    \item[1.] Summarize the \textbf{causal events} and \textbf{outcome event} based on the story and query.
    \begin{itemize}
        \item The causal events should causally contribute to the outcome event.
        \item The causal events should be minimal/atomic and non-overlapping.
    \end{itemize}
    \item[2.] Label the \textbf{occurrences} and \textbf{temporal orders} of events based on the story.
    \begin{itemize}
        \item Label \texttt{occur} as \texttt{true} if an event actually occurs.
        \item Label \texttt{order} as an integer starting from 0, where simultaneous events should share the same \texttt{order}.
    \end{itemize}
    \item[3.] Label the \textbf{focal causal events} based on the query. 
    \begin{itemize}
        \item If the query asks whether a causal event causes the outcome event, the causal event has \texttt{focal = true}.
        \item If the query asks whether an agent causes the outcome event, all causal events reflecting the agent's behaviors have \texttt{focal = true}.
    \end{itemize}
\end{itemize}

Return the causal setting in the following JSON format:
\begin{verbatim}
{
    "causal_events": {
        CAUSAL_EVENT: {
            "occur": True/False,
            "order": ORDER,
            "focal": True/False
        },
        ...
    },
    "outcome_event": {
        OUTCOME_EVENT: {
            "occur": True/False,
            "order": ORDER
        }
    }
}
\end{verbatim}
Return only the JSON, without any extra information.
\end{tcolorbox}
\end{figure*}

\begin{figure*}[h!]
\begin{tcolorbox}[title=Prompt for \ourmodel (Factor Value Inference),width=\textwidth]
Based on the story and causal setting, reason about the values of the following factors for each causal event.

\begin{itemize}
    \item[1.] \texttt{token\_cause = true} if in the story, had a causal event not occurred, the outcome event would not have occurred, while allowing at least a subset of events in the causal setting to remain occurred had a causal event not occurred.
    \item[2.] \texttt{necessary\_cause = true} if in the story, had a causal event not occurred, the outcome event would not have occurred.
    \item[3.] \texttt{sufficient\_cause = true} if in the story, had a causal event occurred, the outcome event would have occurred, even if other causal events had occurred differently.
    \item[*] \texttt{token\_cause = true}, \texttt{necessary\_cause = true}, and \texttt{sufficient\_cause = true} can be satisfied through a path from a causal event to the outcome event, passing through other causal events.
    \item[4.] \texttt{norm\_violated = true} if in the story, a causal event violates norms, such as statistical modes, moral codes, laws, policies, or proper functioning in machines or organisms.
    \item[5.] \texttt{behavior\_intended = true} if in the story, a causal event is an agent's behavior and the agent is aware of the potential consequences of their action and knowingly performs an action that leads to a foreseeable bad outcome.
\end{itemize}

Return the values of factors for causal events in the following JSON format:
\begin{verbatim}
{
    CAUSAL_EVENT: {
        "token_cause": True/False,
        "necessary_cause": True/False,
        "sufficient_cause": True/False,
        "norm_violated": True/False,
        "behavior_intended": True/False
    },
    ...
}
\end{verbatim}
Return only the JSON, without any extra information.
\end{tcolorbox}
\end{figure*}

\begin{figure*}[h!]
\begin{tcolorbox}[title=Prompt for \ourmodel (Determining $f_{\mathrm{resp}}$),width=\textwidth]
\textcolor{blue}{\textbf{Line 9}} \\[0.3em]
Define responsibility as the relative degree (more, less, or equally) to which a causal event causally contributes to the outcome event, relative to other causal events specified. Here, assume responsibility is only determined by normality (\texttt{norm\_violated}) and intention (\texttt{behavior\_intended}).

Return Yes if based on the story, the focal causal event ``\{focal\_event\}'' is equally or more responsible relative to other causal events in the list \{S\_list\}, else No. Then, explain briefly based on the story. \\[0.6em]
\textcolor{blue}{\textbf{Line 15}} \\[0.3em]
Define responsibility as the relative degree (more, less, or equally) to which a causal event causally contributes to the outcome event, relative to other causal events specified. Here, assume responsibility is only determined by temporal order (\texttt{order}).

Return Yes if based on the story, the focal causal event ``\{focal\_event\}'' is more responsible relative to other causal events in the list \{H\_list\}, else No. Then, explain briefly based on the story.
\end{tcolorbox}
\end{figure*}

\begin{figure*}[h!]
\begin{tcolorbox}[title=Prompt for \ourbench (Data Generation),width=\textwidth]
\textcolor{blue}{\textbf{Stage 1}} \\[0.3em]
\textbf{Data Example}: \{data\_example\} \\
In the provided data sample: \texttt{story} is the story background of a logic-based causal judgment problem; \texttt{reasoning} details the reasoning process, including causal/outcome events and their factor values; \texttt{queries} are causality-related queries to the problem, each links to one or more focal causal events in \texttt{reasoning}; \texttt{answers} are Yes/No answers to the queries. \\
Generate a new data example as follows:
\begin{itemize}
    \item[1.] \texttt{story}: Rewrite with a different real-world setting while preserving the reasoning logic, i.e., only rephrase causal/outcome event descriptions.
    \item[2.] \texttt{queries}: Formulate queries for the new story and identify the corresponding focal causal events.
    \item[3.] \texttt{reasoning}: Generate causal/outcome events and their factors for the new story while keeping factor values unchanged.
    \item[4.] \texttt{answers}: Provide Yes/No answers to the queries.
\end{itemize}
Return the new data example in the same JSON format as the original one. \\
Return only the JSON, without any extra information. \\[0.6em]
\textcolor{blue}{\textbf{Stage 2}} \\[0.3em]
Based on the generated data example, refine the new story to make it more distinctive while keeping \texttt{reasoning} unchanged:
\begin{itemize}
    \item[1.] Increase Details: Add relevant details to enhance the new story setting.
    \item[2.] Remove Details: Eliminate elements that are only relevant to the original story.
    \item[3.] Reorganize Structure: Adjust the narrative flow to avoid mirroring the original structure.
\end{itemize}

Return the modified data example in the same JSON format as the original one.

Return only the JSON, without any extra information.
\end{tcolorbox}
\end{figure*}

\end{document}